\documentclass{article}

\usepackage[preprint, nonatbib]{nips_2018}

\usepackage[utf8]{inputenc} 
\usepackage[T1]{fontenc}    
\usepackage{hyperref}       
\usepackage{url}            
\usepackage{booktabs}       
\usepackage{amsfonts}       
\usepackage{nicefrac}       
\usepackage{microtype}      
\usepackage{amsmath}
\usepackage{amssymb}
\usepackage{amsthm}
\usepackage{mathtools}
\usepackage{booktabs}
\usepackage{multirow}
\usepackage{algorithm}
\usepackage{algpseudocode}
\usepackage{comment}
\includecomment{Short}
\excludecomment{Full}
\usepackage[svgnames]{xcolor}
\usepackage{ifthen}
\usepackage{general-private}
\usepackage{special-private}
\usepackage{thm-config}
\usepackage{tikz}
\usetikzlibrary{
calc, intersections}
\usepackage{tkz-euclide}
\usetkzobj{all} 
\usepackage{tikz-private}

\title{Stable Geodesic Update on Hyperbolic Space and  its Application to \Poincare Embeddings}

\author{
  Yosuke Enokida\thanks{equally contributed: mainly contributed to theoretical prototype formula and theoretical analysis} \\
  Graduate School of Information Science and Technology,
  The University of Tokyo \\
  \texttt{enokida.yosuke@ci.i.u-tokyo.ac.jp} \\
  Atsushi Suzuki\thanks{equally contributed: mainly contributed to improvement of stability and numerical experiments} \\
  Graduate School of Information Science and Technology,
  The University of Tokyo \\
  \texttt{atsushi.suzuki.rd@gmail.com} \\
  Kenji Yamanishi \\
  Graduate School of Information Science and Technology,
  The University of Tokyo \\
  \texttt{yamanishi@mist.i.u-tokyo.ac.jp}
}

\includecomment{proceedings}\excludecomment{journal}

\begin{document}

\maketitle

\begin{abstract}
A hyperbolic space has been shown to be more capable of modeling complex networks
than a Euclidean space.
This paper proposes an explicit update rule along geodesics in a hyperbolic space. 
The convergence of our algorithm is theoretically guaranteed, and the convergence rate is better than the conventional Euclidean gradient descent algorithm.
Moreover, our algorithm avoids the ``bias'' problem of existing methods using the Riemannian gradient.
Experimental results demonstrate the good performance of our algorithm in the \Poincare embeddings of knowledge base data.
\end{abstract}

\section{Introduction}
\subsection{Background}

Hyperbolic space is attracting increasing attention in graph embeddings, and has many applications in the field of networks\cite{PhysRevE.82.036106,SustBog,Asta:2015:GNC:3020847.3020859,4215803}, graph theory\cite{6729484}, and visualization\cite{Lamping:1994:LOV:192426.192430,WALTER2004273}.
Recently, Nickel and Kiela \cite{Nickel+:2017} proposed \Poincare embeddings, an algorithm that embeds the nodes $\Vertices$ in a graph $\Graph = \Paren{\Vertices, \Edges}$ into a $\Dimension$-dimensional hyperbolic space $\Hyp^\Dimension$.
The \Poincare embeddings learn a map $\Vertices \ni \LeftNode \mapsto \LeftVec \in \Hyp^\Dimension$ by minimizing the loss function below:
\begin{equation}
\EqnLabel{PoincareLossFunction}
\LossFunc{\Set{\LeftVec}_{\LeftNode \in \Vertices}} \DefEq - \sum_{\LeftNode \in \Vertices}
\sum_{\PositiveNode \in \PositiveNeighborhood{\LeftNode}} \log \frac{\Exponential{- \Distance{\LeftVec}{\PositiveVec}}}{\sum_{\NegativeNode \in \NegativeNeighborhood{\LeftNode}} \Exponential{- \Distance{\LeftVec}{\NegativeVec}}},
\end{equation}
where $\PositiveNeighborhood{\LeftNode} \DefEq \SetBuilder{\PositiveNode}{\Paren{\LeftNode, \PositiveNode} \in \Edges}$ denotes the neighborhood of $\LeftNode$, and $\NegativeNeighborhood{\LeftNode} \DefEq \Vertices \setminus \PositiveNeighborhood{\LeftNode}$ denotes its complement.
The minimization shortens the distance $\Distance{\LeftVec}{\PositiveVec}$ for $\Paren{\LeftNode, \PositiveNode} \in \Edges$, and lengthens the distance $\Distance{\LeftVec}{\NegativeVec}$ for $\Paren{\LeftNode, \NegativeNode} \notin \Edges$.
Thus, the embeddings convert the graph-form-data into vector-form-data, which is applicable for many machine learning methods, without loss of the structure of the graph.
The experimental result in \cite{Nickel+:2017} demonstrated the larger representation capacity of $\Hyp^\Dimension$ than the $\Dimension$-dimensional Euclidean space $\Real^\Dimension$.

The loss function \EqnRef{PoincareLossFunction} of the \Poincare embeddings consists of the distance in $\Hyp^\Dimension$, and its optimization can be interpreted as an optimization problem in $\Hyp^\Dimension$.
Nickel and Kiela \cite{Nickel+:2017} focused on this fact, and used Riemannian gradients instead of Euclidean gradients. 
Their method can be interpreted as a stochastic version of the \emph{natural gradient method} \cite{Amari:1998:NGW:287476.287477}.
All that the natural gradient method requires is the (stochastic) gradients of the function, and thus, it works well even when the number of parameters is very large.
However, its update rule is a move along a "line", in the sense of Euclidean geometry, not a move along a \emph{geodesic}, or the shortest path in $\Hyp^\Dimension$. 
On the other hand, in the field of Riemannian manifold optimization, good update properties along a geodesic have been shown in terms of the conditions for convergence \cite{AbsMahSep2008} \cite{DBLP:journals/tac/Bonnabel13} and convergence rate \cite{pmlr-v49-zhang16b} \cite{NIPS2016_6515}. In this paper, we call updates along a geodesic \emph{geodesic update}.
In general, obtaining a geodesic update in closed form or with small computational complexity is difficult, and no practical algorithm realizing geodesic update in a $\Hyp^\Dimension$ has been proposed, to the best of our knowledge. 
The purpose of this paper is the embodiment of the geodesic update in $\Hyp^\Dimension$.

\subsection{Contribution of This Paper}
We consider general loss functions, ones that consist of the distance in $\Hyp^\Dimension$. 
Let $\Points, \Points' \subset \Hyp^\Dimension$ be finite sets of points in $\Hyp^\Dimension$, and let $\Points' \subset \Points$.
The loss functions that we consider can be written as follows:
\begin{equation}
\EqnLabel{LossFunction}
\LossFunc{\Set{\LeftVec}_{\LeftVec \in \Points'}} \DefEq \tilde{\mathcal{L}} \Paren{\Set{\Distance{\LeftVec}{\PositiveVec}}_{\Paren{\LeftVec, \PositiveVec} \in \Points \times \Points}}.
\end{equation}
Note that \EqnRef{LossFunction} includes the loss function \EqnRef{PoincareLossFunction} of \Poincare embeddings as a special case.
We consider the optimization of the \EqnRef{LossFunction} using its gradients only, because when the number of parameters is large, it is not realistic to obtain information other than its gradient.
We make the following contributions to solving this problem:

\paragraph{a) Derivation of Exponential Map Algorithm and Embodiment of Geodesic Update}
It is necessary to calculate the exponential map in order to realize a geodesic update.
The exponential map is a map that maps a point along a geodesic.
This paper proposes a numerically stable and computationally cheap algorithm to calculate the exponential map in $\Hyp^\Dimension$.
This algorithm realizes the geodesic update in $\Hyp^\Dimension$, which is a special case of the Riemannian gradient descent in \cite{pmlr-v49-zhang16b}.

\paragraph{b) Theoretical Comparison against Euclidean Gradient Descent and Natural Gradient Method}
This paper discusses the theoretical advantages of our update algorithm against the Euclidean gradient update and the natural gradient update. 
We observe that the square distance in $\Hyp^\Dimension$ has worse smoothness as a function in $\Real^\Dimension$ than as a function in $\Hyp^\Dimension$. 
This fact strongly supports the geodesic update against the Euclidean gradient, because the smoothness of the function directly affects the convergence rate.
We also suggest that the natural gradient method has a ``bias'' problem, and does not approach the optimum.
These problems require the natural gradient method to work with a small learning rate, which leads to slow optimization.
Our geodesic update avoids these problems and is stable.

We provide a thorough quantitative analysis on the advantages of our algorithm through the \emph{barycenter} problem. The \emph{barycenter} problem in Riemannian manifolds is attracting growing interest recently\cite{Bijan,ARNAUDON20121437}.
Numerical experiments on the \emph{barycenter} problem and \Poincare embeddings also show the stability of our method and tolerability to a large learning rate, and the instability of the Euclidean update and the natural gradient update.

\subsection{Related Work}

Riemannian optimization is widely applied, for example, in covariance estimation \cite{6298979}, in calculating the Karcher mean of symmetric positive definite matrices \cite{BINI20131700}, in signal processing or image processing\cite{FLETCHER2007250,Pennec2006}, and in statistics \cite{Amari:1998:NGW:287476.287477}. 
The theoretical aspects of Riemannian optimization have also been well studied, for example in \cite{AbsMahSep2008}. Most of the algorithms in \cite{AbsMahSep2008} use \emph{retraction}, a map that approximates the exponential map (a map along a geodesic), instead of calculating the exponential map or geodesics directly.

The geodesic optimization algorithm in a Riemannian manifold is a developing field from both the theoretical and practical aspects. Zhang and Sra \cite{pmlr-v49-zhang16b} analyzed the convergence rate of geodesic update algorithms under some conditions, and numerically showed its performance on the Karcher mean problem of positive semidefinite(PSD) matrices.
Though we have difficulty in calculating a geodesic in general, the idea of coordinate descent is applied to the Lie group of orthogonal matrices \cite{pmlr-v32-shalit14} and achieves certain results. Our method can be thought of as a significant branch of such a practical algorithm.

Stochastic methods using the exponential map have been also studied.
Bonnabel \cite{DBLP:journals/tac/Bonnabel13} analyzed the Riemannian stochastic gradient descent(RSGD), which combines the stochastic gradient descent and retraction in a Riemannian manifold. A variance reduced Riemannian stochastic gradient method was proposed by Zhang \EtAl \cite{NIPS2016_6515}. Calculating the exponential map, which our algorithm facilitates in a hyperbolic space, is a fundamental component of these stochastic methods.

\section{Hyperbolic Space and its Geodesics}

In this section, we introduce a hyperbolic space and its geometry. Although a hyperbolic space is defined as a "Riemannian manifold"\cite{kobayashi1996foundations}, and is well studied in mathematics\cite{ratcliffe2006foundations}, we do not explain the general theory of Riemannian geometry. Instead, we introduce minimal geometrical notions, sufficient to deal with a hyperbolic space.

\subsection{Disk Model of Hyperbolic space}
(The \Poincare disk model of) a hyperbolic space $\Hyp^\Dimension = (\Disk^\Dimension, \HypMetricMatrix)$ consists of a disk $\Disk^\Dimension = \{ \MfdPointP = \Paren{\Range{\DepartureElement^{1}}{\DepartureElement^{2}}{\DepartureElement^{\Dimension}}} \in \Real^\Dimension |\, |\MfdPointP| < \PoleDistance \} $ and a matrix-valued function $H \colon D^\Dimension \ni \MfdPointP \mapsto \HypMetricMatrix(\MfdPointP) \DefEq \HypMetricMatrix_{\MfdPointP} \DefEq \left( \frac{2 \PoleDistance}{\PoleDistance^2 - |\MfdPointP|^2} \right)^2 I_\Dimension$, called the metric of $\Hyp^\Dimension$. Here, $I_\Dimension$ is a unit matrix of size $\Dimension$. The boundary $\partial \Disk^\Dimension$ is called the ideal boundary. 

\begin{definition}
The tangent space $\Hyp^\Dimension$ of $\MfdPointP$, denoted by $T_{\MfdPointP} \Hyp^\Dimension$, is a set of vectors whose foot is at $\MfdPointP$. A vector field $X$ is a function that maps $\MfdPointP \in \Hyp^\Dimension$ to a corresponding tangent vector $X_{\MfdPointP} \in T_{\MfdPointP}\Hyp^\Dimension$.
\end{definition}

The metric plays a role as the {\it ruler} to measure the magnitude of a tangent vector. In a hyperbolic space, the magnitude $\|\TangentVec\|$ of a tangent vector $\TangentVec \in T_{\MfdPointP}\Hyp^\Dimension$ is calculated by $\|\TangentVec\| \DefEq \sqrt{\TangentVec^\top \HypMetricMatrix_{\MfdPointP} \TangentVec}$. 

Notice that a vector on a manifold can be identified with a directional differential operator to a function, or more intuitively, an infinitesimal piece of the curve. Therefore, the derivative of a function $f$ along a vector $\TangentVec$ is defined, which is denoted by ${\TangentVec}f$, indicating an infinitely small change of $f$ in the direction of $\TangentVec$. 

\begin{definition} 
The gradient vector field $ \GradOp{f}$ of a smooth function $f \colon \Hyp^\Dimension \to \Real$ is defined as $(\GradOp{f})_{\MfdPointP} \DefEq \HypMetricMatrix_{\MfdPointP}^{-1} \Vec{\partial} f$, where $\Vec{\partial} f \DefEq (\partial_1 f, \cdots, \partial_\Dimension f)^\top \DefEq (\frac{\partial f}{\partial p^1}, \cdots, \frac{\partial f}{\partial p^\Dimension})^\top$.
\end{definition}
This definition is modified for $\Hyp^\Dimension$. The gradient vector field can be defined for any functions on general Riemannian manifolds, and the general definition coincides with the ordinary gradient vector field in case of $\Real^\Dimension$.
\begin{proceedings}
Using the gradient vector field of $f$, one can define "the gradient flow" of $f$. The value of the function increases along the gradient flow. Therefore, in optimization, it is ideal to calculate the (negative) gradient flow, but this is impossible in most cases. For this reason, we try to approximate the gradient flow by some means.
\end{proceedings}



\subsection{Geodesics and the Exponential Map}

\begin{proceedings}

Although we need some mathematical preliminaries if we want to state the definition of geodesics, in case of a hyperbolic space, we can use a simple characterization that a geodesic is {\it a minimizing curve}. A smooth map $\gamma \colon I \to \Hyp^\Dimension$ defined on an interval $I \subset \Real$ is called a curve on $\Hyp^\Dimension$.
The length $L (\gamma)$ of a curve $\gamma \colon (a,b) \to \Hyp^\Dimension$ is defined by
$
L(\gamma) \DefEq
\int_{a}^{b} \left\| d\gamma/dt \right\|  dt
$.
This definition is a natural extension of the length of a curve in $\Real^\Dimension$.

\begin{definition}
Let $\MfdPointP,\MfdPointQ \in \Hyp^\Dimension$. The shortest curve between $\MfdPointP$ and $\MfdPointQ$ is called the geodesic from $\MfdPointP$ to $\MfdPointQ$.
\end{definition}

A hyperbolic space is known to be "geodesically complete," i.e., there exists a unique geodesic that connects between two arbitrary points in $\Hyp^\Dimension$. Although it is theoretically standard to define a geodesic using the "Levi-Civita connection," the two definitions are equivalent in the case of $\Hyp^\Dimension$.

Mathematically speaking, a geodesic is characterized by an ordinary differential equation system called "geodesic equations." Therefore, if the initial point $x \in M$ and the tangent vector $\TangentVec \in T_{\MfdPointP}M$ are given, there exists a unique geodesic $\gamma_{\TangentVec}$, which satisfies $\gamma_v(0) = \MfdPointP$ and $\dot\gamma_{\TangentVec}(0) = \TangentVec$. Moreover, given a function $f \colon M \to \Real$, one can prove that the geodesic $\gamma_{\TangentVec}$ is a first-order approximation of a gradient flow if $\TangentVec$ comes from the gradient vector field ${\rm grad}\, f$. Therefore, we aim to optimize a function $f$ along geodesics; in other words, we try to calculate the "exponential map."
\end{proceedings}
\begin{definition} The exponential map at $\MfdPointP$ is defined by $\ExpMap_{\MfdPointP}(\TangentVec) \DefEq \gamma_{\TangentVec}(1)$.
\end{definition}

The exponential map moves a point along a geodesic, with an equal distance to the magnitude of the input tangent vector. To construct an algorithm along a geodesic, it is sufficient to solve the geodesic equations to obtain the geodesic $\gamma(t)$ and substitute $t=1$. This is, in general, undesirable due to the difficulty in solving geodesic equations. One of our significant contributions is overcoming this difficulty in the case of hyperbolic spaces, which will be discussed the following section.


\subsection{Difficulties in Calculating the Exponential Map}
One might think that we should try to solve geodesic equations in order to obtain a geodesic or an exponential map in a hyperbolic space. However, this type of strategy does not work.
\begin{proceedings}
Although one can derive the explicit form of the geodesic equations by direct calculation, the result will obtain a variable-coefficient nonlinear differential equation system. 
\end{proceedings}

It is indeed difficult to solve the equations of geodesics directly and obtain an explicit form of geodesics, but an implicit form of geodesics in a hyperbolic space is given based on the properties of the isometry group in the disk model of a hyperbolic space.
In other words, the properties of the isometry group give us the following characteristics of the geodesics in a hyperbolic space, which are sufficient to determine a geodesic:
\begin{lemma}
In the disk model of a hyperbolic space, (i) a curve is a geodesic if and only if it is a segment of a circle or line which intersects with the ideal boundary at right angles, and (ii) the distance between $\MfdPointP, \MfdPointQ \in \Disk^\Dimension$ is given by
\begin{equation}
\Distance{\MfdPointP}{\MfdPointQ} = {\rm arcosh} \left(1+ 2\frac{\PoleDistance^2 |\MfdPointP-\MfdPointQ|^2}{(\PoleDistance^2 - |\MfdPointP|^2)
(\PoleDistance^2 - |\MfdPointQ|^2)}\right).
\end{equation}
\end{lemma}
For a proof, see p.126 and p.123 of \cite{ratcliffe2006foundations}.

\subsection{Explicit form of Exponential Map}
In the following discussion, we obtain an explicit form of geodesics and exponential maps using the characteristics of geodesics. Suppose that we are given a smooth function $f \colon \Hyp^\Dimension \to \Real$ and considering the optimization problem of $f$. Our aim is to derive an explicit form of
$\ExpMap_{\DepartureVec} \Paren{- \GradientTangentVec}$,
given a point $\DepartureVec \in \Hyp^\Dimension$
and the gradient $\GradientTangentVec = (\GradOp{f})_{\MfdPointP} = \HypMetricMatrix_{\DepartureVec}^{-1}\GradientVec \in T_{\MfdPointP}\Hyp^\Dimension$, where $\GradientVec$ denotes the directional derivatives $\GradientVec \DefEq \Vec{\partial} f = (\partial_1 f, \partial_2 f, \cdots, \partial_\Dimension f)^\top$ of $f$.
Since geodesics are only circles that intersect with $\partial \Hyp^\Dimension$ at right angles, we can explicitly calculate the exponential map given a tangent vector using an elementary geometry. 
The naive way to numerically obtain the exponential map is to obtain the orthonormal bases $\CBrace{\EXVec, \EYVec}$ of the plane spanned by $\DepartureVec$ and $\GradientVec$, and calculate the intersection of the two ``circles'' (the geodesic and equidistance curve). Thus, if $\DepartureVec$ and $\GradientVec$ are linearly independent, we can obtain the following form:
\begin{equation}
\EqnLabel{UnstableExponentialMap}
\ExpMap_{\DepartureVec} \Paren{- \GradientTangentVec} - \DepartureVec
 = \GradientComponent \EXVec + \NormalComponent \EYVec,
\end{equation}
where $x$ and $y$ depend on $\DepartureVec$ and $\GradientVec$. See the supplementary material for the specific form.

However, this kind of formula does not work in numerical experiments.
When $\DepartureVec$ and $\GradientVec$ are almost linearly dependent, the orthonormal bases $\CBrace{\EXVec, \EYVec}$ are numerically unstable. 
Moreover, in this situation, the radius of the geodesic circle is close to infinity and it also causes numerical instability in obtaining the geodesic circle explicitly.
We can avoid these problems by arranging 
\EqnRef{UnstableExponentialMap}
so that it is tolerant to limit operation, to obtain the following theorem. Let $\Sinc$ denote the cardinal sine function.
\begin{theorem}
\ThmLabel{ExponentialMap}
Let $\GradientTangentVec \in T_{\MfdPointP}\Hyp^\Dimension$ be a tangent vector. Let $\GradientVec \DefEq \HypMetricMatrix_{\DepartureVec} \GradientTangentVec$, $\ArrivalDistance \DefEq \GradientVec^\top \HypMetricMatrix_{\DepartureVec}^{-1}\GradientVec$,
$\DepartureDistance \DefEq \Abs{\DepartureVec} \DefEq \sqrt{\sum_{i=1}^{\Dimension} \Paren{\DepartureElement^i}^2}$, 
$\FTerm \DefEq \GradientVec \cdot \DepartureVec \DefEq \sum_{i=1}^{\Dimension} \GradientElement_i \DepartureElement^i$, 
and $\CoshDistanceMinusOne \DefEq \cosh \ArrivalDistance - 1$.  
Then,
\vspace{-2mm}
\begin{equation}
\begin{split}
\ExpMap_{\DepartureVec} \Paren{- \GradientTangentVec} - \DepartureVec
& = \Paren{\SHTerm^2 \STTerm \SXiTerm - \frac{2 \SHTerm^2 \FTerm \STTerm^2 \SXiTerm^2}{1 + \sqrt{1 - 4 \DepartureDistance^2 \CoshDistanceMinusOne \SXiTerm^2 + 4 \FTerm^2 \STTerm^2 \SXiTerm^2}}} \GradientVec 
+ \frac{2 \SHTerm^2 \CoshDistanceMinusOne \SXiTerm^2}{1 + \sqrt{1 - 4 \DepartureDistance^2 \CoshDistanceMinusOne \SXiTerm^2 + 4 \FTerm^2 \STTerm^2 \SXiTerm^2}} \DepartureVec.
\end{split}
\end{equation}
\vspace{-2mm}
where
\begin{equation}
\begin{split}
\SHTerm^2
\DefEq
\PoleDistance^2 - \DepartureDistance^2, \quad
\SZTerm^2
\DefEq
2 \PoleDistance^2 + \CoshDistanceMinusOne \Paren{\PoleDistance^2 + \DepartureDistance^2} - 2 \FTerm^2 \STTerm^2, \quad
\STTerm 
\DefEq
\frac{\SHTerm^2}{2 \PoleDistance \sqrt{\cosh \ArrivalDistance + 1}} \Sinc \frac{\ArrivalDistance}{\sqrt{-1} \PiUnit},
\end{split}
\end{equation}
and
\begin{equation}
\begin{split}
& \SXiTerm
=
\frac{- \FTerm \STTerm \Bracket{\SZTerm^2 - 2 \CoshDistanceMinusOne \DepartureDistance^2 + 2 \FTerm^2 \STTerm^2} - \SZTerm^2 \sqrt{\Bracket{\SZTerm^2 - \CoshDistanceMinusOne \DepartureDistance^2 + 2 \FTerm^2 \STTerm^2}} }{4 \DepartureDistance^2 \CoshDistanceMinusOne \FTerm^2 \STTerm^2 - 4 \FTerm^4 \STTerm^4 + \SZTerm^4}.
\end{split}
\end{equation}
\end{theorem}

\begin{remark}
The unstable parts in \EqnRef{UnstableExponentialMap} are reduced to the $\Sinc$ function in \ThmRef{ExponentialMap}. Therefore, computation of the intermediate variables in \ThmRef{ExponentialMap} are stable with stable implementation of the $\Sinc$ function. 
For the same reason, \ThmRef{ExponentialMap} is applicable even if $\DepartureVec$ and $\GradientVec$ are linearly dependent.
\end{remark}

\begin{remark}
The computational cost of the formula in \ThmRef{ExponentialMap} with respect to dimensionality $\Dimension$ is $O(\Dimension)$, which has the same order as that of the gradient calculation. Hence, the computational cost in \ThmRef{ExponentialMap} is equal to that of the \emph{natural gradient update} \cite{Amari:1998:NGW:287476.287477} up to a constant factor.
\end{remark}

See the supplementary material for a proof. Using \ThmRef{ExponentialMap}, we can realize the Riemannian gradient descent \cite{pmlr-v49-zhang16b} in a hyperbolic space. The right pseudo-code and figure in \ref{Methods} show the algorithm. Here, the robustness of \ThmRef{ExponentialMap} to the linear dependency of $\DepartureVec$ and $\GradientVec$ is important, because $\GradientVec$ is very small in the gradient descent setting.

\section{Theoretical Analysis}

In this section, we discuss the theoretical advantage of our method against the Euclidean gradient update and the natural gradient update, shown in the left and center of Figure \ref{Methods}. For simplicity, we assume that the radius of the disk model $\PoleDistance$ is 1 in this section.

\subsection{Comparison with Euclidean Gradient}

In this subsection, we compare our exponential map method and the Euclidean gradient descent method. 
To compare the rate of convergence, we mainly consider $\mu$-strongly and $L$-smooth function. This setting is popular in the optimization of Riemannian manifolds. 
\begin{proceedings}
\begin{definition}
A function $f \colon \Hyp^\Dimension \to \Real$ is called $\HypMetricMatrix_{\MfdPointP}$-geodesically $\mu$-strongly convex if $| f(\ExpMap_{\MfdPointP}(\TangentVec)) - f(\MfdPointP) - \TangentVec^\top \HypMetricMatrix_{\MfdPointP} ({\rm grad}\, f)_{\MfdPointP} | \geq \frac{\mu}{2} \|\TangentVec\|^2$ holds for any $\MfdPointP \in \Hyp^\Dimension$ and $\TangentVec \in T_{\MfdPointP}\Hyp^\Dimension$. $f$ is called $\HypMetricMatrix_{\MfdPointP}$-geodesically $L$-smooth if $| f(\ExpMap_{\MfdPointP}(\TangentVec)) - f(\MfdPointP) - \TangentVec^\top \HypMetricMatrix_{\MfdPointP} ({\rm grad}\, f)_{\MfdPointP} | \leq \frac{L}{2} \|\TangentVec\|^2$ holds for any $\MfdPointP \in \Hyp^\Dimension$ and $\TangentVec \in T_{\MfdPointP}\Hyp^\Dimension$.
\end{definition}
\end{proceedings}


We notice that this definition is an extension of the standard definition of strongly convexity or smoothness on $\Real^\Dimension$. \cite{pmlr-v49-zhang16b} showed that for a geodesically $\mu$-convex $L$-smooth function, the geodesic update converges with rate $O((1-\frac{\mu}{L})^t)$. 
Note that $\mu$ and $L$ depend on the metric; in other words, the metric determines the convergence rate.
The following example shows that the geodesic update, the method based on the hyperbolic metric can have a significant advantage than the Euclidean gradient update, the method based on the Euclidean metric, when we consider a function of the hyperbolic distance. 



\subsubsection{Example: Barycenter problem}

In this subsection, as an example of our theoretical analysis, we focus on the barycenter problem, or Karcher mean problem. 
The barycenter problem corresponds to the numerator of \EqnRef{PoincareLossFunction}, but is easier to analyze. Moreover, the problem itself is interesting in terms of embeddings because the barycenter can be interpreted as the conceptional center of entities. We show that the barycenter problem can be solved with an exponential rate.
Let $\MfdPointQ_1, \cdots, \MfdPointQ_n \in \Hyp^\Dimension$. The barycenter problem is to calculate 
\begin{equation}
\EqnLabel{Barycenter}
\MfdPointP_{opt} \in {\rm argmin}_{\MfdPointP} \frac{1}{n}\sum_i \SqDistance{\MfdPointP}{\MfdPointQ_i},
\end{equation}

First, we focus on the squared distance.
\begin{proposition}
\label{SquareDistance}
Let $K \subset \Hyp^\Dimension$ be a compact set that includes the origin, and $f(\MfdPointP) \DefEq \SqDistance{0}{\MfdPointP}$. Then $f$ is $H_{\MfdPointP}$-geodesically 1-strongly convex and $[\max_{\MfdPointP \in K}\Distance{0}{\MfdPointP} {\rm coth}\, \Distance{0}{\MfdPointP}]$-smooth. 
\end{proposition}
This proposition shows that the smoothness $L$ of a squared distance is almost proportional to the distance $\Distance{0}{\MfdPointP}$ if we take account of the Riemannian structure. On the other hand, $L$ is larger than $e^{\Distance{0}{\MfdPointP}}$ if we forget the structure. 

The objective function of \EqnRef{Barycenter} is known to be 1-strongly convex. Although the squared distance is not $L$-smooth in general setting, we can find a compact set $K$ in which the generated sequence remains, and restriction of $f(\MfdPointP) = \frac{1}{n}\sum_i \SqDistance{\MfdPointP}{\MfdPointQ_i}$ to $K$ is $L$-smooth for a sufficiently large $L$. 
\begin{proceedings}
To prove the smoothness of \EqnRef{Barycenter}, we again take advantage of the Riemannian hessian.
\begin{lemma}\label{smoothness_lemma}
Let $K \subset \Hyp^\Dimension$ be a compact set, $k_1 = \max_{\MfdPointR \in K} \{\Distance{0}{\MfdPointR}\}$, and $k_2 = \max_i\{\Distance{0}{\MfdPointQ_i} \}$. Then, the function $K \ni \MfdPointP \mapsto \frac{1}{n}\sum_i \SqDistance{\MfdPointP}{\MfdPointQ_i}$ is $(k_1+ k_2 + 1)$-smooth.
\end{lemma}
\end{proceedings}
\begin{theorem}
Let $\MfdPointP_0$ be an initial point and $D = \max \{\Distance{0}{\MfdPointP_0},k_2\}$. Then, the sequence $\{\MfdPointP_i\}$ generated with constant step size $\eta = 1/(2D + 1)$ remains inside the compact set $K_D = \{ \MfdPointR \in \Hyp^\Dimension\, |\, \Distance{0}{\MfdPointR} \leq D \}$, and satisfies 
$f(\MfdPointP_t) - f(\MfdPointP_{opt}) \leq  (1-\varepsilon)^{t-2} D^3$, where 
$\varepsilon = \min \{ 1/(D\, {\rm coth}\, D), 1/(2D+1) \} $.
\end{theorem}
On the other hand, the following proposition holds with respect to the (hyperbolic) squared distance in terms of the Euclidean metric:
\begin{proposition}\label{EuclideanRatio}
Let $f(\MfdPointP) \DefEq \SqDistance{0}{\MfdPointP}$. If we regard $f$ as a function from $\Real^\Dimension$ to $\Real$, $f$ is $8$-strongly convex and $[\max_{\MfdPointP \in K}(\cosh(\Distance{0}{\MfdPointP})-1)\frac{4\Distance{0}{\MfdPointP} + 1/p}{1-p^2}]$-smooth.
\end{proposition}
Therefore, the ratio $\frac{\mu}{L}$ of \EqnRef{Barycenter} can be much worse, when we forget the Riemannian structure.
These fact give the geodesic update a significant advantage against the Euclidean gradient descent.

\subsection{``Bias'' Problem of Natural Gradient Method}
The so-called "natural gradient" method \cite{Amari:1998:NGW:287476.287477} is widely used in Riemannian optimization problems. 
These methods use Riemannian gradient vectors instead of Euclidean gradient vectors. 
However, the natural gradient does not use geodesics, but updates by simply adding a gradient vector to the original point. See Figure \ref{Methods} (center).
Notice that we cannot add a point and a tangent vector without embedding a manifold to some Euclidean space.
Although the natural gradient update approximates the geodesic update with a low learning rate, the difference between them is significant with a high learning rate.
Moreover, we can conclude that the natural gradient does not converge to an optimal point, even in quite a simple situation. To show this, we work on the following question.

\begin{problem}
 Let $\Hyp^1 = \{ \MfdPointP \in \Real\, | \, |\MfdPointP| <1\}$ be a disk model of 1-dim hyperbolic space and $\varepsilon \in (0,1)$. We are given $\MfdPointQ_0 =0 \in \Hyp^1$ and $\MfdPointQ_1 = 1-\varepsilon \in \Hyp^1$. Solve the barycenter problem, i.e., calculate ${\rm argmin}_{\MfdPointP} \SqDistance{\MfdPointP}{\MfdPointQ_0} + \SqDistance{\MfdPointP}{\MfdPointQ_1}. $
\end{problem}

Intuitively, the answer must be a "hyperbolic middle point," in other words, the optimal point must satisfy $2\Distance{0}{\MfdPointP_{opt}}  =\Distance{0}{1-\varepsilon}$. This intuition is correct.
Put $f_0 = \frac{1}{2}\SqDistance{\MfdPointP}{\MfdPointQ_0}$ and $f_1 = \frac{1}{2}\SqDistance{\MfdPointP}{\MfdPointQ_1}$. Now, suppose we are trying to solve this example question via the natural gradient method and geodesic method in figure \ref{Methods}. The oracle $\tilde{\nabla}_t$ is $\Vec{\partial} f_0$ or $\Vec{\partial} f_1$, with probability 1/2 each.
According to the theorem below, the expected variation from the optimal point is 0 in the geodesic case, and is not 0 in the natural gradient case. This shows that our method is balanced at the optimal, while the natural gradient is {\it biased}.

\begin{theorem}\label{GeodesicBarycenter}
Put $\MfdPointP_l \DefEq \ExpMap_{\MfdPointP_{opt}}(-\StepSize\, {\rm grad}\, f_0)$ and $\MfdPointP_r \DefEq \ExpMap_{\MfdPointP_{opt}}(-\StepSize\, {\rm grad}\, f_1)$. Then, $\Distance{\MfdPointP_{opt}}{\MfdPointP_l} = \Distance{\MfdPointP_{opt}}{\MfdPointP_r}$. 
\end{theorem}
\begin{theorem}\label{NaturalBarycenter}
Put $\MfdPointR_l \DefEq \MfdPointP_{opt} -  \StepSize\, {\rm grad}\, f_0$ and $\MfdPointR_r \DefEq \MfdPointP_{opt} - \StepSize\, {\rm grad}\, f_1$. Then, $\Distance{\MfdPointP_{opt}}{\MfdPointR_l} < \Distance{\MfdPointP_{opt}}{\MfdPointR_r}$. 
\end{theorem}
\begin{proceedings}
We can prove the former theorem from the properties of the exponential map, and for the latter part, we explicitly calculate the coordinate of $\MfdPointR_l$ and $\MfdPointR_r$ as
\begin{equation}
\MfdPointR_l = \MfdPointP_{opt} - \StepSize \frac{\sqrt{1-|\MfdPointP_{opt}|^2}}{2}f(\MfdPointP_{opt}), \MfdPointR_r = \MfdPointP_{opt} + \eta \frac{\sqrt{1-|\MfdPointP_{opt}|^2}}{2}[f(1-\varepsilon) - f(\MfdPointP_{opt})],
\end{equation}
and comparing them with $\MfdPointP_l$ and $\MfdPointP_r$ leads to this theorem. See the supplementary material for a complete proof.
\end{proceedings}

\newcommand{\GradFigureScale}{0.40}
\begin{figure}[htbp]
 \begin{minipage}{0.33\hsize}
  \begin{algorithm}[H]
    \begin{algorithmic}
    \caption{Euclidean GU} 
    \State $\DepartureVec^{(0)} \gets \DepartureVec_\mathrm{initial}$
    \For{$t = \Range{0}{1}{T-1}$}
      \State $\begin{cases}\GradientVec^{(t)} \gets \nabla f (\DepartureVec^{(t)}) \\ \GradientVec^{(t)} \gets \tilde{\nabla}_{t} \end{cases}$
      \State $\DepartureVec^{(t+1)} \gets \DepartureVec^{(t)} - \LearningRate_{t} \GradientVec^{(t)}$
    \EndFor
    \State \Return $\DepartureVec^{(T)}$
    \State
    \State
    \end{algorithmic}
  \end{algorithm}
 \end{minipage}
 \begin{minipage}{0.33\hsize}
  \begin{algorithm}[H]
    \begin{algorithmic}
    \caption{Natural GU} 
    \State $\DepartureVec^{(0)} \gets \DepartureVec_\mathrm{initial}$
    \For{$t = \Range{0}{1}{T-1}$}
      \State $\begin{cases}\GradientVec^{(t)} \gets \nabla f (\DepartureVec^{(t)}) \\ \GradientVec^{(t)} \gets \tilde{\nabla}_{t} \end{cases}$
      \State $\GradientTangentVec^{(t)} \gets \HypMetricMatrix^{-1} \GradientVec^{(t)}$
      \State $\DepartureVec^{(t+1)} \gets \DepartureVec^{(t)} - \LearningRate_{t} \GradientTangentVec^{(t)}$
    \EndFor
    \State \Return $\DepartureVec^{(T)}$
    \State
    \end{algorithmic}
  \end{algorithm}
 \end{minipage}
 \begin{minipage}{0.33\hsize}
  \begin{algorithm}[H]
    \begin{algorithmic}
    \caption{Geodesic U} 
    \State $\DepartureVec^{(0)} \gets \DepartureVec_\mathrm{initial}$
    \For{$t = \Range{0}{1}{T-1}$}
      \State $\begin{cases}\GradientVec^{(t)} \gets \nabla f (\DepartureVec^{(t)}) \\ \GradientVec^{(t)} \gets \tilde{\nabla}_{t} \end{cases}$
      \State $\GradientTangentVec^{(t)} \gets \HypMetricMatrix^{-1} \GradientVec^{(t)}$
      \State $\DepartureVec^{(t+1)} \gets$
      \State $\quad \ExpMap_{\DepartureVec^{(t)}}(-\LearningRate_{t} \GradientTangentVec^{(t)})$
    \EndFor
    \State \Return $\DepartureVec^{(T)}$
    \end{algorithmic}
  \end{algorithm}
 \end{minipage}
 \begin{minipage}{0.33\hsize}
  \begin{center}
   \includegraphics[keepaspectratio=true, scale = \GradFigureScale]{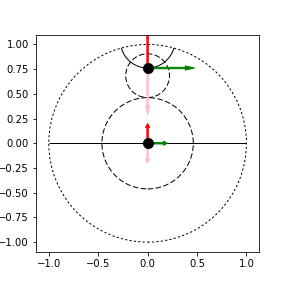}
  \end{center}
 \end{minipage}
 \begin{minipage}{0.33\hsize}
  \begin{center}
   \includegraphics[keepaspectratio=true, scale = \GradFigureScale]{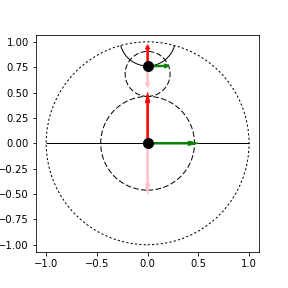}
  \end{center}
 \end{minipage}
 \begin{minipage}{0.33\hsize}
  \begin{center}
   \includegraphics[keepaspectratio=true, scale = \GradFigureScale]{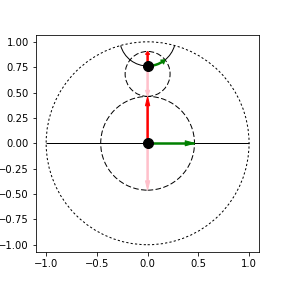}
  \end{center}
 \end{minipage}
 
 \caption{Euclidean gradient update (left: Euclidean GU), natural gradient update (center: Natural GU), geodesic update (right: Geodesic U): Pseudo codes (upper) and behaviors (lower). The upper case in the code describes deterministic methods and the lower case describes stochastic methods, where $\tilde{\nabla}_{t}$ denotes the stochastic oracle, which is expected to satisfy $\Expect{\tilde{\nabla}_{t}} = \nabla f (\DepartureVec^{(t)})$. The arrows in the figures show the update rule with a gradient from the points indicated by the black dots. The magnitude in the sense of the Riemannian metric of each gradient is 0.01 in the left figure, and 1.0 in the center and right figure. The solid lines are geodesics, and the dashed lines indicate the equidistant curves from black points. The Euclidean gradient does not reflect the scale in a hyperbolic plane. Although the natural gradient reflects the scale, the update result is not on geodesics. Moreover, it causes an overrun when the negative gradient outward is given and vice versa. This causes the ``bias'' problem. The geodesic update strictly reflects the magnitude of the gradient.}
\label{Methods}
\end{figure}

\newcommand{\LossRatio}{0.20}
\newcommand{\LossScale}{0.36}
\newcommand{\HistRatio}{0.12}
\newcommand{\HistScale}{0.36}

\begin{figure}[htbp]
 \begin{minipage}{\LossRatio\hsize}
  \begin{center}
   \includegraphics[keepaspectratio=true, scale = \LossScale]{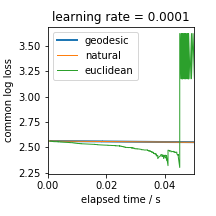}
  \end{center}
 \end{minipage}
 \begin{minipage}{\HistRatio\hsize}
   \centering
   \includegraphics[keepaspectratio=true, scale = \HistScale]{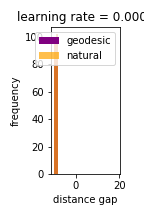}
 \end{minipage}
 \begin{minipage}{\LossRatio\hsize}
  \begin{center}
   \includegraphics[keepaspectratio=true, scale = \LossScale]{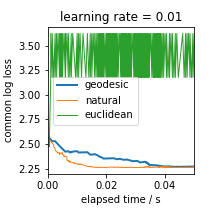}
  \end{center}
 \end{minipage}
 \begin{minipage}{\HistRatio\hsize}
   \centering
   \includegraphics[keepaspectratio=true, scale = \HistScale]{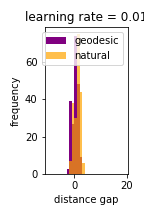}
 \end{minipage}
 \begin{minipage}{\LossRatio\hsize}
  \begin{center}
   \includegraphics[keepaspectratio=true, scale = \LossScale]{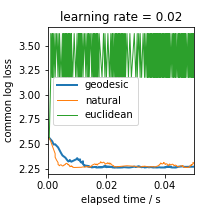}
  \end{center}
 \end{minipage}
 \begin{minipage}{\HistRatio\hsize}
   \centering
   \includegraphics[keepaspectratio=true, scale = \HistScale]{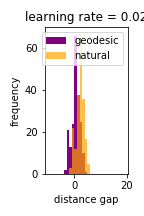}
 \end{minipage}
 \begin{minipage}{\LossRatio\hsize}
  \begin{center}
   \includegraphics[keepaspectratio=true, scale = \LossScale]{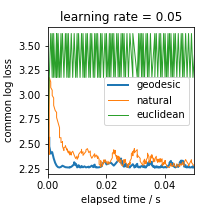}
  \end{center}
 \end{minipage}
 \begin{minipage}{\HistRatio\hsize}
   \centering
   \includegraphics[keepaspectratio=true, scale = \HistScale]{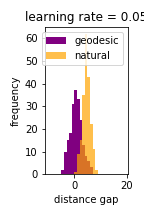}
 \end{minipage}
 \begin{minipage}{\LossRatio\hsize}
  \begin{center}
   \includegraphics[keepaspectratio=true, scale = \LossScale]{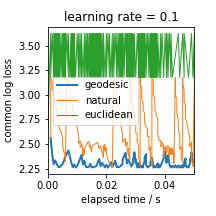}
  \end{center}
 \end{minipage}
 \begin{minipage}{\HistRatio\hsize}
   \centering
   \includegraphics[keepaspectratio=true, scale = \HistScale]{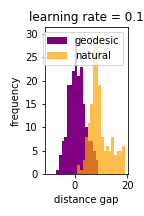}
 \end{minipage}
 \begin{minipage}{\LossRatio\hsize}
  \begin{center}
   \includegraphics[keepaspectratio=true, scale = \LossScale]{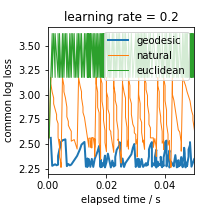}
  \end{center}
 \end{minipage}
 \begin{minipage}{\HistRatio\hsize}
   \centering
   \includegraphics[keepaspectratio=true, scale = \HistScale]{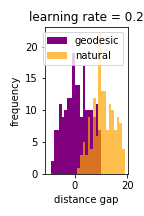}
 \end{minipage}
 \caption{Barycenter problem (left: transition of common log loss, right: histogram of updated points in the last 200 iterations): The number zero in the histogram indicates $\MfdPointP_{opt}$, and  apositive value corresponds to the outward direction. With a higher learning rate, the natural gradient update failed to minimize the loss function, whereas the geodesic update succeeded in minimizing the same. The histogram shows that the higher learning rate, the more serious the outward ``bias'' problem is. This is why the natural gradient update failed. The Euclidean method failed even with an extremely small learning rate.}
 \FigLabel{Barycenter}
\end{figure}

\newcommand{\ArtificialRatio}{0.24}
\newcommand{\ArtificialScale}{0.32}
\newcommand{\RealRatio}{0.48}
\newcommand{\RealScale}{0.32}

\begin{figure}[htbp]
  \begin{minipage}{0.67\hsize}
  \begin{minipage}{\ArtificialRatio\hsize}
    \centering
    \includegraphics[keepaspectratio=true, scale = \ArtificialScale]{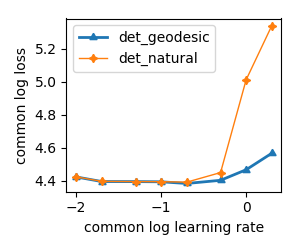}
  \end{minipage}
  \begin{minipage}{\ArtificialRatio\hsize}
    \centering
    \includegraphics[keepaspectratio=true, scale = \ArtificialScale]{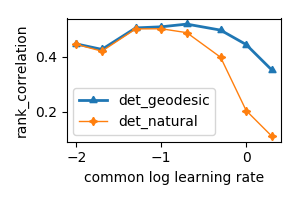}
  \end{minipage}
  \begin{minipage}{\ArtificialRatio\hsize}
    \centering
    \includegraphics[keepaspectratio=true, scale = \ArtificialScale]{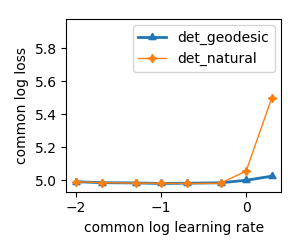}
  \end{minipage}
  \begin{minipage}{\ArtificialRatio\hsize}
    \centering
    \includegraphics[keepaspectratio=true, scale = \ArtificialScale]{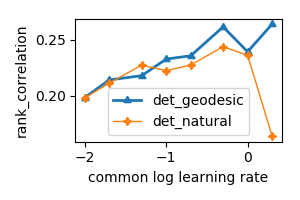}
  \end{minipage}
  \caption{Poincare embeddings (artificial data): loss function (the mean of the last 100 iterations) and Kendall's rank correlation coefficient in \Poincare embeddings problem (left: the undirected complete binary tree (depth: 5), right: the directed complete binary tree with its transitive closure (depth: 5)). The proposed method is stable even with a high learning rate.}
  \FigLabel{Artificial}
  \end{minipage}
  \begin{minipage}{0.33\hsize}
  \begin{minipage}{\RealRatio\hsize}
    \centering
    \includegraphics[keepaspectratio=true, scale = \RealScale]{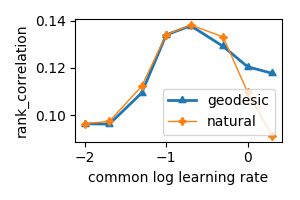}
  \end{minipage}
  \begin{minipage}{\RealRatio\hsize}
    \centering
    \includegraphics[keepaspectratio=true, scale = \RealScale]{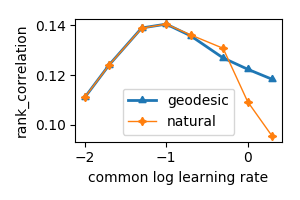}
  \end{minipage}
  \caption{Poincare embeddings (real data): Kendall's rank correlation coefficient in \Poincare embeddings problem (left: embeddings to $\Hyp^2$, right: embeddings to $\Hyp^5$).}
  \FigLabel{Real}
  \end{minipage}
\end{figure}


\section{Experiments}
\subsection{Barycenter Problem}
First, we evaluated the performance of the proposed update rule on a barycenter problem with artificial data. 
We compared the Euclidean gradient update, the natural gradient update \cite{Amari:1998:NGW:287476.287477}, and the proposed method.
We fixed two points (0, 0), (1 - 1e-8, 0) on $\Hyp^2$ and calculated the barycenter by using stochastic gradient descent methods (Sample size: 2, batch size: 1). We compared each method with learning rate 0.0001, 0.01, 0.02, 0.05, 0.1, and 0.2. \FigRef{Barycenter} shows the transition of the loss function and the histogram of the position of the points in the last 200 iterations. 
When the learning rate was high, the natural gradient update failed to minimize the loss function, whereas the exponential map update succeeded in minimizing the loss function. 
The histogram shows that the natural gradient update tended to move the points outward from the optimum; in other words, it suffered from the ``bias'' problem. This is why the natural gradient update failed. 
On the other hand, the natural gradient update worked faster with a low learning rate. This is due to the constant factor of the computational cost (Note that the difference between the geodesic update and the natural gradient update is small with a low learning rate). 
The result shows that the proposed algorithm works correctly even with a high learning rate, and it is expected to obtain the solution faster with a higher learning rate compared with the natural gradient method.
The Euclidean method failed even with an extremely small learning rate (Note that the dimension of the learning rate is different in the Euclidean update and the other two update rules, and thus, we evaluated them with an extremely small learning rate). This is because the gradient in the Euclidean metric diverges near the ideal boundary.

\subsection{\Poincare Embedding}
We evaluated the proposed geodesic update in \Poincare embeddings \cite{Nickel+:2017} for minimizing the loss function \EqnRef{PoincareLossFunction}.
\begin{proceedings}
As artificial data, we used a complete binary tree (depth: 5). We used both of the undirected tree and the directed tree with its transitive closure as in \cite{Nickel+:2017}. 
As real data, we used the noun subset of \emph{WordNet}'s hypernymy relations \cite{WordNet:2010} (subset the root of which is \emph{mammal}). See the supplementary material for details.
Here, we applied the proposed (stochastic) geodesic update rule and the (stochastic) natural gradient method implemented in \emph{gensim} \cite{Rehurek:2010}, and evaluated their performance and robustness to changes in the learning rate.
While we did not use the negative sampling in the artificial data experiment to optimize the loss function strictly, we used the negative sampling in the real data experiment, since its data size was large.
\FigRef{Artificial} shows the result in the artificial data.
The figure shows the loss function and Kendall's rank correlation coefficient \cite{Kendall:1938} \cite{Kendall:1945} of the distance matrix in the graph and hyperbolic space. 
The more accurate the structure preserved by the embeddings, the higher the value of the coefficient.
\FigRef{Real} shows the result in the real data, though we have to take the effect of the negative sampling into consideration.
As these figures show, the natural gradient method is vulnerable to changes in the learning rate, whereas the proposed method is stable.
\end{proceedings}

\section{Conclusion and Future Work}
We have proposed a geodesic update rule on hyperbolic spaces.
The proposed algorithm considers the metric in a hyperbolic space as well as the natural gradient method. Moreover, the proposed method is stable compared with the natural gradient method. One significant branch of future studies is a combination of our methods and other techniques available in the context of Riemannian optimization.
For example, we expect we can combine the proposed update rule with Riemannian acceleration methods as in \cite{Liu:2017} and accelerating the proposed method will further increase the quality of embeddings.
General notions of Riemannian optimization are well studied, and we furthermore discussed the properties of optimization methods focusing on hyperbolic spaces. We expect we can further work on hyperbolic optimization taking into consideration a simple structure of $\Hyp^\Dimension$, as we constructed a simple algorithm for $\Hyp^\Dimension$ using a special characterization of geodesics on it.

\bibliographystyle{plain}
\bibliography{ref}
\newpage
\def\thesection{\Alph{section}}
\setcounter{section}{0}
\section{Appendix : Derivation of Geodesic Update}

\newcommand{\DepartureRatio}{0.65}
\newcommand{\VertexRatio}{1.5}
\newcommand{\GradientRatio}{0.2}
\newcommand{\DiskRadius}{60mm}
\newlength{\DepartureRadius}
\setlength{\DepartureRadius}{\DiskRadius*\real{\DepartureRatio}}
\newlength{\VertexRadius}
\setlength{\VertexRadius}{\DiskRadius*\real{\VertexRatio}}
\newlength{\GradientLength}
\setlength{\GradientLength}{\DiskRadius*\real{\GradientRatio}}
\newcommand{\PositivePoleAngle}{0}
\newcommand{\NegativePoleAngle}{\PositivePoleAngle + 180}
\newcommand{\DepartureAngle}{\PositivePoleAngle + 65}
\newcommand{\VertexAngle}{\PositivePoleAngle + 80}

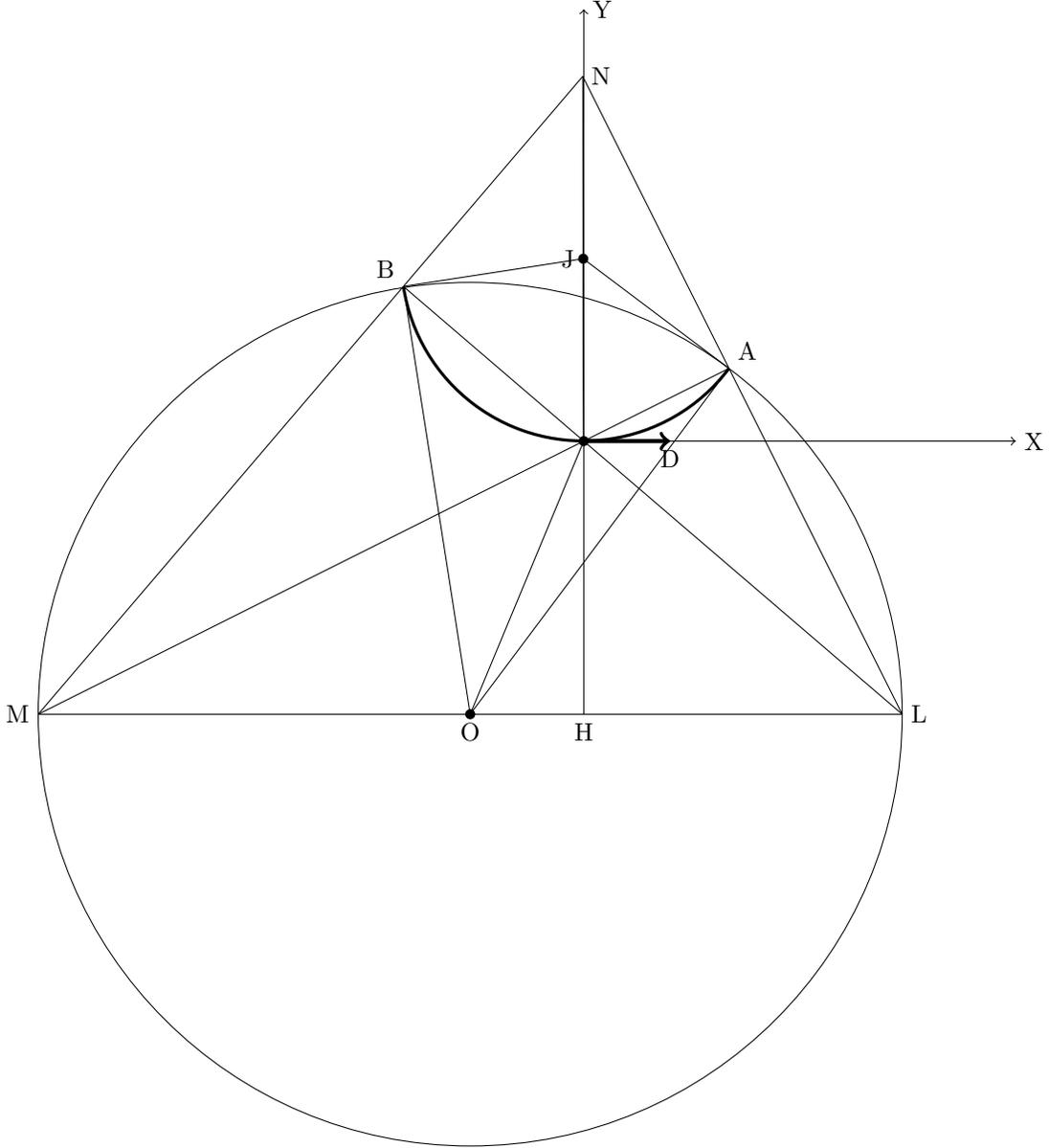
\begin{figure}
\begin{tikzpicture}
\SetPoint{\Origin}{0}{0}
\DotPoint{\Origin}
\DrawCircle{\Origin}{\DiskRadius}

\SetRadius{\PositivePole}{\Origin}{\PositivePoleAngle}{\DiskRadius}
\SetRadius{\NegativePole}{\Origin}{\NegativePoleAngle}{\DiskRadius}

\SetPoint{\NorthVertex}{\VertexAngle}{\VertexRadius}

\Connect{\PositivePole}{\NorthVertex}
\Connect{\NegativePole}{\NorthVertex}

\SetFoot[name path=PositiveLine]{\PositiveIntersection}{\PositivePole}{\NegativePole}{\NorthVertex}
\SetFoot[name path=NegativeLine]{\NegativeIntersection}{\NegativePole}{\PositivePole}{\NorthVertex}
\SetIntersection{\Departure}{PositiveLine}{NegativeLine}
\DotPoint{\Departure}
\Connect{\Origin}{\Departure}
\Connect{\Departure}{\NorthVertex}
\SetRadius[->, ultra thick]{\GradientIndicator}{\Departure}{\PositivePoleAngle}{\GradientLength}
\SetRadius[->]{\PositiveNormalizedGradient}{\Departure}{\PositivePoleAngle}{\DiskRadius}

\SetDivisor{\GeodesicCenter}{\Departure}{0.5}{\NorthVertex}
\DotPoint{\GeodesicCenter}
\Connect{\GeodesicCenter}{\PositiveIntersection}
\Connect{\GeodesicCenter}{\NegativeIntersection}

\Connect{\Origin}{\PositiveIntersection}
\Connect{\Origin}{\NegativeIntersection}
\SetFoot{\DepartureShadow}{\NegativePole}{\Departure}{\PositivePole}

\PGFExtractAngle{\PositiveIntersectionAngle}{\GeodesicCenter}{\PositiveIntersection}
\PGFExtractAngle{\NegativeIntersectionAngle}{\GeodesicCenter}{\NegativeIntersection}
\DrawArc[very thick]{\PositiveIntersection}{\GeodesicCenter}{\NegativeIntersection}

\SetRadius[->]{\GradientNormal}{\Departure}{\PositivePoleAngle+90}{\DiskRadius}

\LabelPoint{\Origin}{\Origin}{below}
\LabelPoint{\PositivePole}{\PositivePole}{right}
\LabelPoint{\NegativePole}{\NegativePole}{left}
\LabelPoint{\NorthVertex}{\NorthVertex}{right}
\LabelPoint{\PositiveIntersection}{\PositiveIntersection}{above right}
\LabelPoint{\NegativeIntersection}{\NegativeIntersection}{above left}
\LabelPoint{\GradientIndicator}{\GradientIndicator}{below}
\LabelPoint{\PositiveNormalizedGradient}{\PositiveNormalizedGradient}{right}
\LabelPoint{\GeodesicCenter}{\GeodesicCenter}{left}
\LabelPoint{\DepartureShadow}{\DepartureShadow}{below}
\LabelPoint{\GradientNormal}{\GradientNormal}{right}
\end{tikzpicture}
\caption{\Poincare disk model and a geodesic. Point $\Point{\Departure}$ denotes the point to be updated and arrow $\Segment{\Departure}{\GradientIndicator}$ denotes vector or the gradient of the loss function. The geodesic is given as an arc $\Point{\PositiveIntersection \Departure \NegativeIntersection}$ of a circle, the center of which is denoted by $\Point{\GeodesicCenter}$. The circle is determined by a triangle, the orthocenter of which is $\Point{\Departure}$.}
\end{figure}

In this section, we prove \ThmRef{ExponentialMap}.
In the following discussion, let $\Point{\Departure}$ and $\Point{\Arrival}$ denote the points that $\DepartureVec$ and $\ArrivalVec$ indicate.

\subsection{Geodesic and its Curvature}
In this subsection, we obtain the geodesic that passes through the point to be updated with the gradient of the loss function as the tangent vector.
With the disk model of a hyperbolic space, a geodesic is given by an arc, or a part of a circle orthogonal to the boundary of the disk (hyperball). Here, the arc passes through $\Point{\Departure}$ and is tangent to $\SegmentVec{\Departure}{\GradientIndicator}$.
Let $\Point{\Origin}$ denote the center of the unit disk that is identified with the hyperbolic space and let $\PoleDistance$ denote its radius.
Let $\Point{\Departure}$ denote the point to be updated and $\SegmentVec{\Departure}{\GradientIndicator}$ denote the gradient of the loss function.
The geodesic that passes through $\Point{\Departure}$ with tangent vector $\SegmentVec{\Departure}{\GradientIndicator}$ is obtained by the following lemma:
\begin{lemma}
\LemmaLabel{UnstableExponentialMap}
Assume that $\SegmentVec{\Departure}{\GradientIndicator}$ is not parallel to $\Segment{\Origin}{\Departure}$.
\begin{enumerate}
\item Let $\Point{\PositivePole}$ and $\Point{\NegativePole}$ be points that satisfies $\SegmentVec{\Origin}{\PositivePole} = \PoleDistance \frac{\SegmentVec{\Departure}{\GradientIndicator}}{\Abs{\SegmentVec{\Departure}{\GradientIndicator}}}$ and $\SegmentVec{\Origin}{\NegativePole} = - \PoleDistance \frac{\SegmentVec{\Departure}{\GradientIndicator}}{\Abs{\SegmentVec{\Departure}{\GradientIndicator}}}$, respectively.

\item Let $\Point{\PositiveIntersection}$ be the intersection of unit circle $\Point{\Origin}$ and line $\Segment{\NegativePole}{\Departure}$ (which is not $\Point{\NegativePole}$), and let $\Point{\NegativeIntersection}$ be the intersection of unit circle $\Point{\Origin}$ and line $\Segment{\PositivePole}{\Departure}$ (which is not $\Point{\PositivePole}$), likewise.

\item Let $\Point{\NorthVertex}$ be the intersection of line $\Segment{\PositivePole}{\PositiveIntersection}$ and $\Segment{\NegativePole}{\NegativeIntersection}$.

\item Let $\Point{\GeodesicCenter}$ be the middle point of segment $\Segment{\Departure}{\NorthVertex}$.
\end{enumerate}
Then, the arc that passes through $\Point{\PositiveIntersection}$, $\Point{\Departure}$ and $\Point{\NegativeIntersection}$ is the geodesic on which $\SegmentVec{\Departure}{\GradientIndicator}$ lies, and segment $\Segment{\Departure}{\NorthVertex}$ is a diameter of the circle which contains the arc and the center of the circle (arc) is point $\Point{\GeodesicCenter}$, the middle point of $\Segment{\Departure}{\NorthVertex}$. 
In other words, the arc is tangent to $\Segment{\Departure}{\GradientIndicator}$ at point $\Point{\Departure}$ and orthogonal to the circle $\Point{\Origin}$ with $\Point{\PositiveIntersection}$ and $\Point{\NegativeIntersection}$ as the two intersections.
\end{lemma}
\begin{proof}
Since $\Segment{\NegativePole}{\PositivePole}$ is a diameter of the hyperball, we have $\Segment{\NegativePole}{\PositiveIntersection} \perp \Segment{\PositivePole}{\NorthVertex}$ and $\Segment{\PositivePole}{\NegativeIntersection} \perp \Segment{\NegativePole}{\NorthVertex}$. 
Now, $\Triangle{\Departure}{\PositiveIntersection}{\NorthVertex}$ and $\Triangle{\Departure}{\NegativeIntersection}{\NorthVertex}$ are right triangles.
Therefore, the points $\Point{\Departure}$, $\Point{\PositiveIntersection}$, $\Point{\NorthVertex}$, and $\Point{\NegativeIntersection}$ are on the circle, the center of which is point $\Point{\GeodesicCenter}$, the middle point of $\Segment{\Departure}{\NorthVertex}$.
Moreover, point $\Point{\Departure}$, which is the intersection of $\Segment{\NegativePole}{\PositiveIntersection}$ and $\Segment{\PositivePole}{\NegativeIntersection}$, is the orthocenter of $\Triangle{\NorthVertex}{\PositivePole}{\NegativePole}$.
Hence, we have $\Segment{\Departure}{\NorthVertex} \perp \Segment{\NegativePole}{\PositivePole}$, which suggests that the circle that passes through $\Point{\Departure}$, $\Point{\PositiveIntersection}$, $\Point{\NorthVertex}$, and $\Point{\NegativeIntersection}$ is tangent to $\Segment{\Departure}{\GradientIndicator}$ at point $\Point{\Departure}$.

Now, we prove $\Segment{\Origin}{\PositiveIntersection} \perp \Segment{\GeodesicCenter}{\PositiveIntersection}$ and $\Segment{\Origin}{\NegativeIntersection} \perp \Segment{\GeodesicCenter}{\NegativeIntersection}$ below.
Since $\Segment{\GeodesicCenter}{\Departure} = \Segment{\GeodesicCenter}{\PositiveIntersection}$, we have $\Angle{\GeodesicCenter}{\PositiveIntersection}{\Departure} = \Angle{\GeodesicCenter}{\Departure}{\PositiveIntersection}$. 
Let $\Point{\DepartureShadow}$ be the intersection of line $\Segment{\NegativePole}{\PositivePole}$ and $\Segment{\NorthVertex}{\Departure}$.
Note that since point $\Point{\Departure}$ is the orthocenter of $\Triangle{\NorthVertex}{\PositivePole}{\NegativePole}$, we have $\Segment{\NegativePole}{\PositivePole} \perp \Segment{\DepartureShadow}{\NorthVertex}$.
Now, because both of $\Angle{\GeodesicCenter}{\Departure}{\PositiveIntersection} = \Angle{\NorthVertex}{\Departure}{\PositiveIntersection}$ and $\Angle{\Origin}{\PositivePole}{\PositiveIntersection} = \Angle{\DepartureShadow}{\PositiveIntersection}{\NorthVertex}$ are complementary angles of $\Angle{\Departure}{\NorthVertex}{\PositiveIntersection}$, these are equal.
Since $\Segment{\Origin}{\PositivePole} = \Segment{\Origin}{\PositiveIntersection}$, we have $\Angle{\Origin}{\PositivePole}{\PositiveIntersection} = \Angle{\Origin}{\PositiveIntersection}{\PositivePole}$.
Therefore, we get $\Angle{\GeodesicCenter}{\PositiveIntersection}{\Departure} = \Angle{\Origin}{\PositiveIntersection}{\PositivePole}$.
Hence, we obtain $\Angle{\Origin}{\PositiveIntersection}{\GeodesicCenter} = \Angle{\GeodesicCenter}{\PositiveIntersection}{\Departure} + \Angle{\Origin}{\PositiveIntersection}{\Departure} = \Angle{\Origin}{\PositiveIntersection}{\PositivePole} + \Angle{\Origin}{\PositiveIntersection}{\Departure} = \Angle{\NegativePole}{\PositiveIntersection}{\PositivePole} = 90^{\circ}$, that is, $\Segment{\Origin}{\PositiveIntersection} \perp \Segment{\GeodesicCenter}{\PositiveIntersection}$. We can also prove $\Segment{\Origin}{\NegativeIntersection} \perp \Segment{\GeodesicCenter}{\NegativeIntersection}$. These suggests that circle $\Point{O}$ and $\Point{J}$ are orthogonal.
\end{proof}

We obtain the center $\Point{\GeodesicCenter}$ of the geodesic arc and its radius $\GeodesicRadius$ and $\GeodesicCurvature$ by vector operations below:
Let $\PoleVec \DefEq \SegmentVec{\Origin}{\PositivePole}$, $\DepartureVec \DefEq \SegmentVec{\Origin}{\Departure}$ and $\NorthVec \DefEq \SegmentVec{\Origin}{\NorthVertex}$, and let $\PoleDistance \DefEq \Abs{\PoleVec}$, $\DepartureDistance \DefEq \Abs{\DepartureVec}$, and $\TwoVecProduct^2 \DefEq \PoleVec \cdot \DepartureVec$.
Note that though $\PoleVec \cdot \DepartureVec$ can be negative, it does not lose the discussion below.
We can obtain $\NorthVec$ as follows:
\begin{lemma}
\LemmaLabel{North}
Assume that $\SegmentVec{\Departure}{\GradientIndicator}$ is not parallel to $\Segment{\Origin}{\Departure}$.
Then 
\begin{equation}
\NorthVec \DefEq \SegmentVec{\Origin}{\NorthVertex}
=
\frac{\TwoVecProduct^2 \Paren{\DepartureDistance^2 - \PoleDistance^2}}{\DepartureDistance^2 \PoleDistance^2 - \TwoVecProduct^4} \PoleVec + \frac{\DepartureDistance^4 - \TwoVecProduct^4}{\DepartureDistance^2 \PoleDistance^2 - \TwoVecProduct^4} \DepartureVec.
\end{equation}
\end{lemma}
\begin{proof}
Since $\Point{\NorthVertex}$ lies on the plane on which $\Point{\Origin}$, $\Point{\PositivePole}$, and $\Point{\Departure}$ lie.
Hence, there exist $\PoleComponent, \DepartureComponent \in \Real$ such that $\NorthVec = \PoleComponent \PoleVec + \DepartureComponent \DepartureVec$.
Because $\Point{\Departure}$ is the orthocenter of the $\Triangle{\PositivePole}{\NorthVertex}{\NegativePole}$, we get $\SegmentVec{\NegativePole}{\Departure} \perp \SegmentVec{\PositivePole}{\NorthVertex}$ and $\SegmentVec{\PositivePole}{\Departure} \perp \SegmentVec{\NegativePole}{\NorthVertex}$.
Hence, the following holds.
\begin{equation}
\begin{split}
\Paren{\DepartureVec - \PoleVec} \cdot \Paren{\NorthVec + \PoleVec} 
& =
0,
\\
\Paren{\DepartureVec + \PoleVec} \cdot \Paren{\NorthVec - \PoleVec} 
& =
0.
\end{split}
\end{equation}
Substituting $\NorthVec = \PoleComponent \PoleVec + \DepartureComponent \DepartureVec$, we have
\begin{equation}
\begin{split}
\Paren{\TwoVecProduct^2 - \PoleDistance^2} \PoleComponent + \Paren{\DepartureDistance^2 - \TwoVecProduct^2} \DepartureComponent + \Paren{\TwoVecProduct^2 - \PoleDistance^2} 
& =
0,
\\
\Paren{\TwoVecProduct^2 + \PoleDistance^2} \PoleComponent + \Paren{\DepartureDistance^2 + \TwoVecProduct^2} \DepartureComponent - \Paren{\TwoVecProduct^2 + \PoleDistance^2} 
& =
0.
\end{split}
\end{equation}
Solving this equation, we have
\begin{equation}
\begin{split}
\PoleComponent 
& =
\frac{\TwoVecProduct^2 \Paren{\DepartureDistance^2 - \PoleDistance^2}}{\DepartureDistance^2 \PoleDistance^2 - \TwoVecProduct^4},
\\
\DepartureComponent
& =
\frac{\DepartureDistance^4 - \TwoVecProduct^4}{\DepartureDistance^2 \PoleDistance^2 - \TwoVecProduct^4},
\end{split}
\end{equation}
which completes the proof.
\end{proof}
Using this lemma, we can obtain the curvature of the geodesic arc.
\begin{lemma}
\LemmaLabel{GeodesicCurvature}
The curvature $\GeodesicCurvature$ satisfies the following:
\begin{equation}
\GeodesicCurvature^2 = \frac{4 \Paren{\DepartureDistance^2 - \frac{\TwoVecProduct^4}{\PoleDistance^2}}}{\Paren{\DepartureDistance^2 - \PoleDistance^2}^2}.
\end{equation}
\end{lemma}
\begin{remark}
\LemmaRef{GeodesicCurvature} holds even if $\SegmentVec{\Departure}{\GradientIndicator}$ is parallel to $\Segment{\Origin}{\Departure}$.
In this case, the curvature is 0, that is, the geodesic is a Euclidean line.
\end{remark}
\begin{proof}
If $\SegmentVec{\Departure}{\GradientIndicator}$ is parallel to $\Segment{\Origin}{\Departure}$, the both hand sides of the equation are equal to 0, which satisfies the equation.
We discuss below the case in which $\SegmentVec{\Departure}{\GradientIndicator}$ is not parallel to $\Segment{\Origin}{\Departure}$.
Segment $\Segment{\Departure}{\NorthVertex}$ is a diameter of the geodesic.
Hence the radius $\GeodesicRadius$ of the geodesic is given by $\frac{1}{2} \Abs{\SegmentVec{\Departure}{\NorthVertex}} = \Abs{\NorthVec - \DepartureVec}$.
Now, we have
\begin{equation}
\begin{split}
\GeodesicRadius^2
& =
\frac{1}{4} \Paren{\NorthVec - \DepartureVec} \cdot \Paren{\NorthVec - \DepartureVec}.
\end{split}
\end{equation}
By \LemmaRef{North}, we have
\begin{equation}
\begin{split}
\NorthVec - \DepartureVec
& =
\frac{\DepartureDistance^2 - \PoleDistance^2}{\DepartureDistance^2 \PoleDistance^2 - \TwoVecProduct^4} \Paren{\TwoVecProduct^2 \PoleVec - \PoleDistance^2 \DepartureVec}.
\end{split}
\end{equation}
Therefore, we obtain
\begin{equation}
\GeodesicRadius^2 = \frac{\PoleDistance^2 \Paren{\DepartureDistance^2 - \PoleDistance^2}^2}{4 \Paren{\DepartureDistance^2 \PoleDistance^2 - \TwoVecProduct^4}}.
\end{equation}
Taking the inverse of the both sides of the equation, we complete the proof.
\end{proof}

\subsection{Equidistance Curve in Hyperbolic Space}
In this subsection, we obtain the equidistance curve from the point to be updated. 
Here, equidistance curve from a point with distance $\ArrivalDistance$ is defined as the set of the points, the distance of which from the point is equal to $\ArrivalDistance$.
In this section, we measure the distance with the hyperbolic metric. 

\begin{lemma}
\LemmaLabel{Equidistance}
Let $\CoshDistanceMinusOne \DefEq \cosh \ArrivalDistance - 1$. 
The equidistance curve from $\Point{\Departure}$ with the distance $\ArrivalDistance$ is given by the circle, the center $\Point{\EquidistanceCenter}$ of which is given by 
\begin{equation}
\SegmentVec{\Origin}{\EquidistanceCenter}
=
\frac{2 \PoleDistance^2}{2 \PoleDistance^2 + \CoshDistanceMinusOne \Paren{\PoleDistance^2 - \DepartureDistance^2}} \DepartureVec,
\end{equation}
and the radius $\EquidistanceRadius$ of which is given by
\begin{equation}
\EquidistanceRadius^2
=
\frac{\CoshDistanceMinusOne \Paren{\CoshDistanceMinusOne + 2}\PoleDistance^2 \Paren{\PoleDistance^2 - \DepartureDistance^2}^2}{\Bracket{2 \PoleDistance^2 + \CoshDistanceMinusOne \Paren{\PoleDistance^2 - \DepartureDistance^2}}^2}
\end{equation}
\end{lemma}
\begin{proof}
Let $\Point{\Arrival}$ be a point that lies on the equidistance curve. The distance $\ArrivalDistance$ of $\Point{\Arrival}$ from $\Point{\Departure}$ satisfies the following:
\begin{equation}
\CoshDistanceMinusOne \DefEq \cosh \ArrivalDistance - 1 = \frac{2 \Segment{\Origin}{\PositivePole}^2 \Segment{\Departure}{\Arrival}^2}{\Paren{\Segment{\Departure}{\PositivePole}^2 - \Segment{\Origin}{\Departure}^2}\Paren{\Segment{\Departure}{\PositivePole}^2 - \Segment{\Origin}{\Arrival}^2}}.
\end{equation}
Now, we have
\begin{equation}
\CoshDistanceMinusOne = \frac{2 \Abs{\PoleVec}^2 \Abs{\ArrivalVec - \DepartureVec}^2}{\Paren{\Abs{\PoleVec}^2 - \Abs{\DepartureVec}^2} \Paren{\Abs{\PoleVec}^2 - \Abs{\ArrivalVec}^2}}.
\end{equation}
Thus, the following holds:
\begin{equation}
\Abs{\ArrivalVec}^2 - 2 \frac{2 \PoleDistance^2}{2 \PoleDistance^2 + \CoshDistanceMinusOne \Paren{\PoleDistance^2 - \DepartureDistance^2}} \Paren{\ArrivalVec \cdot \DepartureVec} + \frac{2 \DepartureDistance^2 - \CoshDistanceMinusOne \Paren{\PoleDistance^2 - \DepartureDistance^2}^2}{\Bracket{2 \PoleDistance^2 + \CoshDistanceMinusOne \Paren{\PoleDistance^2 - \DepartureDistance^2}}^2} \PoleDistance^2 = 0
\end{equation}
By completing the square, we get the following:
\begin{equation}
\Abs{\ArrivalVec - \frac{2 \PoleDistance^2}{2 \PoleDistance^2 + \CoshDistanceMinusOne \Paren{\PoleDistance^2 - \DepartureDistance^2}} \DepartureVec}^2 = \frac{\CoshDistanceMinusOne \Paren{\CoshDistanceMinusOne + 2}\PoleDistance^2 \Paren{\PoleDistance^2 - \DepartureDistance^2}^2}{\Bracket{2 \PoleDistance^2 + \CoshDistanceMinusOne \Paren{\PoleDistance^2 - \DepartureDistance^2}}^2},
\end{equation}
which completes the proof.
\end{proof}

\subsection{Exponential Map}
In this subsection, we complete the proof of \ThmRef{ExponentialMap} using the results in previous subsections.
When tangent vector $\GradientTangentVec \in T_{\Point{\Departure}} M$ is given, the exponential map $\Exponential{- \GradientTangentVec}$ returns the point $\Point{\Arrival}$ that satisfies
\begin{itemize}
\item $\Distance{\Point{\Departure}}{\Point{\Arrival}} = \ArrivalDistance = \GradientTangentVec^\top \HypMetricMatrix_{\Point{\Departure}} \GradientTangentVec = \GradientVec^\top \HypMetricMatrix_{\Point{\Departure}}^{-1}\GradientVec$, where $\GradientVec \DefEq \HypMetricMatrix_{\Point{\Departure}} \GradientTangentVec$.
\item In the disk model, there exists a circle or line such that 1) it passes through $\Point{\Departure}$ and $\Point{\Arrival}$, 2) it is tangent to $\GradientVec$ at $\Point{\Departure}$ and 3) the inner product of $\GradientVec \cdot \SegmentVec{\Departure}{\Arrival} \le 0$.
\end{itemize}

Therefore, if the geodesic is given by a circle in the disk model, we can obtain the exponential map using \LemmaRef{Equidistance} and the radius $\GeodesicRadius$ or curvature $\GeodesicCurvature$ of the geodesic. 
The destination $\Point{\Arrival}$ of the update from point $\Point{\Departure}$ with gradient vector $\SegmentVec{\Departure}{\GradientIndicator}$ is given as follows:
\begin{lemma}
Let $\GradientTangentVec \in T_{\Point{\Departure}}\Hyp^\Dimension$ be a tangent vector. Let $\GradientVec \DefEq \HypMetricMatrix_{\Point{\Departure}} \GradientTangentVec$, $\ArrivalDistance \DefEq \GradientVec^\top \HypMetricMatrix_{\Point{\Departure}}^{-1}\GradientVec$,
$\DepartureDistance \DefEq \Abs{\DepartureVec} \DefEq \sqrt{\sum_{i=1}^{\Dimension} \Paren{\DepartureElement^i}^2}$, 
$\FTerm \DefEq \GradientVec \cdot \DepartureVec \DefEq \sum_{i=1}^{\Dimension} \GradientElement_i \DepartureElement^i$, 
and $\CoshDistanceMinusOne \DefEq \cosh \ArrivalDistance - 1$.  
Define $\EXVec$ by $\EXVec = \frac{\GradientVec}{\Abs{\GradientVec}}$. 
Let $\EYVec$ be a numerical vector such that it satisfies $\Abs{\EYVec} = 1$, $\EXVec \perp \EYVec$, $\DepartureVec \cdot \EYVec \ge 0$, and $\DepartureVec$ is a linear combination of $\EXVec$ and $\EYVec$.
Then, the exponential map is given by $\ExpMap_{\Point{\Departure}} \Paren{- \GradientTangentVec} = \Point{\Arrival}$ where point $\Point{\Arrival}$ satisfies
\begin{equation}
\begin{split}
\SegmentVec{\Departure}{\Arrival} 
& = \GradientComponent \EXVec + \NormalComponent \EYVec.
\end{split}
\end{equation}
where
\begin{equation}
\begin{split}
\KX
& = 
- \frac{\CoshDistanceMinusOne \Paren{\PoleDistance^2 - \DepartureDistance^2}}{2 \PoleDistance^2 + \CoshDistanceMinusOne \Paren{\PoleDistance^2 - \DepartureDistance^2}}
\frac{\TwoVecProduct^2}{\PoleDistance}
\\
\KY
& =
- \frac{\CoshDistanceMinusOne \Paren{\PoleDistance^2 - \DepartureDistance^2}}{2 \PoleDistance^2 + \CoshDistanceMinusOne \Paren{\PoleDistance^2 - \DepartureDistance^2}}
\sqrt{\DepartureDistance^2 - \frac{\TwoVecProduct^4}{\PoleDistance^2}}. 
\\
\GeodesicRadius 
& = \sqrt{\frac{\PoleDistance^2 \Paren{\DepartureDistance^2 - \PoleDistance^2}^2}{4 \Paren{\DepartureDistance^2 \PoleDistance^2 - \TwoVecProduct^4}}}
\\
\SPsiTerm^{2}
& = \GeodesicRadius \Paren{\GeodesicRadius - \KY}
\\
\SPTerm^{2} 
& =
\EquidistanceRadius^2 - \Paren{\KX^2 + \KY^2},
\end{split}
\end{equation}
and
\begin{equation}
\begin{split}
\GradientComponent 
&= \frac{\KX \Bracket{2 \SPsiTerm^{2} - \SPTerm^{2}} - \sqrt{4 \KX^2 \SPsiTerm^{4} + 4 \frac{\SPsiTerm^{6} \SPTerm^{2}}{\GeodesicRadius^2} - \frac{\SPsiTerm^{4} \SPTerm^{4}}{\GeodesicRadius^2}}}{2 \Bracket{\KX^2 + \frac{\SPsiTerm^{4}}{\GeodesicRadius^2}}} \\
\NormalComponent 
&= \GeodesicRadius - \sqrt{\GeodesicRadius^2 - \GradientComponent^2}.
\end{split}
\end{equation}
\end{lemma}

\begin{proof}

Let the destination of the update be denoted by $\Point{\Arrival}$, and $\SegmentVec{\Departure}{\Arrival} = \GradientComponent \EXVec + \NormalComponent \EYVec$.
Since $\Point{\Arrival}$ is located on the geodesic, it satisfies the following:
\begin{equation}
\NormalComponent = \GeodesicRadius - \sqrt{\GeodesicRadius^2 - \GradientComponent^2},
\end{equation}
where $\GeodesicRadius = \sqrt{\frac{\PoleDistance^2 \Paren{\DepartureDistance^2 - \PoleDistance^2}^2}{4 \Paren{\DepartureDistance^2 \PoleDistance^2 - \TwoVecProduct^4}}}$ is the radius of the geodesic given by \LemmaRef{GeodesicCurvature}.
Let $\KX$ and $\KY$ denote $\EXVec$ and $\EYVec$ component of $\SegmentVec{\Departure}{\EquidistanceCenter}$, respectively. 
Here, it holds that $\SegmentVec{\Departure}{\EquidistanceCenter} = \KX \EXVec + \KY \EYVec$.
Note that the following holds:
\begin{equation}
\begin{split}
\SegmentVec{\Departure}{\EquidistanceCenter}
& =
\SegmentVec{\Origin}{\EquidistanceCenter} - \SegmentVec{\Origin}{\Departure}
\\
& =
- \frac{\CoshDistanceMinusOne \Paren{\PoleDistance^2 - \DepartureDistance^2}}{2 \PoleDistance^2 + \CoshDistanceMinusOne \Paren{\PoleDistance^2 - \DepartureDistance^2}} \DepartureVec,
\end{split}
\end{equation}
and we have
\begin{equation}
\begin{split}
\KX
& =
\SegmentVec{\Departure}{\EquidistanceCenter} \cdot \frac{\PoleVec}{\PoleDistance}
\\
& = 
- \frac{\CoshDistanceMinusOne \Paren{\PoleDistance^2 - \DepartureDistance^2}}{2 \PoleDistance^2 + \CoshDistanceMinusOne \Paren{\PoleDistance^2 - \DepartureDistance^2}}
\frac{\TwoVecProduct^2}{\PoleDistance}
\\
\KY
& = - \sqrt{\SegmentVec{\Departure}{\EquidistanceCenter} \cdot \SegmentVec{\Departure}{\EquidistanceCenter} - \Paren{\SegmentVec{\Departure}{\EquidistanceCenter} \cdot \frac{\PoleVec}{\PoleDistance}}^2}
\\
& =
- \frac{\CoshDistanceMinusOne \Paren{\PoleDistance^2 - \DepartureDistance^2}}{2 \PoleDistance^2 + \CoshDistanceMinusOne \Paren{\PoleDistance^2 - \DepartureDistance^2}}
\sqrt{\DepartureDistance^2 - \frac{\TwoVecProduct^4}{\PoleDistance^2}}.
\end{split}
\end{equation}

Since $\Point{\Arrival}$ is located on the equidistance curve from $\Point{\Departure}$, it satisfies the following:
\begin{equation}
\Paren{\GradientComponent^2 - \KX}^2 + \Paren{\NormalComponent^2 - \KY}^2 = \EquidistanceRadius^2.
\end{equation}
We can calculate $\GradientComponent$ as follows:
\begin{equation}
\Paren{\GradientComponent - \KX}^2 + \Paren{\Paren{\GeodesicRadius - \sqrt{\GeodesicRadius^2 - \GradientComponent^2}} - \KY}^2
=
\EquidistanceRadius^2.
\end{equation}
Hence,
\begin{equation}
\Paren{\GradientComponent - \KX}^2 + \Paren{\GeodesicRadius - \KY}^2 - 2 \Paren{\GeodesicRadius - \KY} \sqrt{\GeodesicRadius^2 - \GradientComponent^2} + \Paren{\GeodesicRadius^2 - \GradientComponent^2}
=
\EquidistanceRadius^2.
\end{equation}
Now, we have
\begin{equation}
- 2 \KX \GradientComponent + 2 \GeodesicRadius \Paren{\GeodesicRadius - \KY} - \Paren{\EquidistanceRadius^2 - \Paren{\KX^2 + \KY^2}}
= 
2 \Paren{\GeodesicRadius - \KY} \sqrt{\GeodesicRadius^2 - \GradientComponent^2}. \end{equation}
By taking the square of the both hand side, we have
\begin{equation}
\begin{split}
& 4 \KX^2 \GradientComponent^2 - 4 \KX \GradientComponent \Bracket{2 \GeodesicRadius \Paren{\GeodesicRadius - \KY} + \Paren{\EquidistanceRadius^2 - \Paren{\KX^2 + \KY^2}}} 
\\ 
& + \Bracket{2 \GeodesicRadius \Paren{\GeodesicRadius - \KY} - \Paren{\EquidistanceRadius^2 - \Paren{\KX^2 + \KY^2}}}^2 \\
& = 
4 \Paren{\GeodesicRadius - \KY}^2 \Paren{\GeodesicRadius^2 - \GradientComponent^2}.
\end{split}
\end{equation}
Hence we get the following quadratic equation:
\begin{equation}
\begin{split}
& 4 \Bracket{\KX^2 + \Paren{\GeodesicRadius - \KY}^2} \GradientComponent^2 - 4 \KX \GradientComponent \Bracket{2 \GeodesicRadius \Paren{\GeodesicRadius - \KY} - \Paren{\EquidistanceRadius^2 - \Paren{\KX^2 + \KY^2}}} \\
& - 4 \GeodesicRadius \Paren{\GeodesicRadius - \KY} \Paren{\EquidistanceRadius^2 - \Paren{\KX^2 + \KY^2}} + \Paren{\EquidistanceRadius^2 - \Paren{\KX^2 + \KY^2}}^2 \\
& = 
0.
\end{split}
\end{equation}
Now, we define $\SPsiTerm^{2}$ and $\SPTerm^{2}$ by $\SPsiTerm^{2} \DefEq \GeodesicRadius \Paren{\GeodesicRadius - \KY}$ and $\SPTerm^{2} \DefEq \EquidistanceRadius^2 - \Paren{\KX^2 + \KY^2}$.
Using these variables, the quadratic equation is written as follows:
\begin{equation}
\begin{split}
& 4 \Bracket{\KX^2 + \frac{\SPsiTerm^{4}}{\GeodesicRadius^2}} \GradientComponent^2 - 4 \KX \Bracket{2 \SPsiTerm^{2} - \SPTerm^{2}} \GradientComponent - 4 \SPsiTerm^{2} \SPTerm^{2} + \SPTerm^{4} = 0.
\end{split}
\end{equation}
The solution is given by the following:
\begin{equation}
\begin{split}
\EqnLabel{GradientComponentUnstable}
\GradientComponent = \frac{\KX \Bracket{2 \SPsiTerm^{2} - \SPTerm^{2}} - \sqrt{4 \KX^2 \SPsiTerm^{4} + 4 \frac{\SPsiTerm^{6} \SPTerm^{2}}{\GeodesicRadius^2} - \frac{\SPsiTerm^{4} \SPTerm^{4}}{\GeodesicRadius^2}}}{2 \Bracket{\KX^2 + \frac{\SPsiTerm^{4}}{\GeodesicRadius^2}}}
\end{split}
\end{equation}
Calculating $\NormalComponent$ by $\NormalComponent = \GeodesicRadius - \sqrt{\GeodesicRadius^2 - \GradientComponent^2}$ completes the proof
\end{proof}
Although \LemmaRef{UnstableExponentialMap} gives the exponential map in most cases, some symbols in the lemma diverges to infinity in special cases, which causes fatal numerical instability.
Indeed, if $\GradientVec$ is zero, we cannot determine $\EXVec$ and $\EYVec$ uniquely, and if $\GradientVec$ is extremely close to zero, $\EXVec$ and $\EYVec$ are numerically unstable.
Even if $\GradientVec$ is non-zero, if $\GradientVec$ is parallel to $\DepartureVec$, we cannot determine $\EYVec$ uniquely and $\GeodesicRadius$ diverges to infinity, and if $\GradientVec$ is almost parallel to $\DepartureVec$, $\NormalComponent$ is numerically unstable.
To avoid these problems, we construct \ThmRef{ExponentialMap}, the formula consists of $\DepartureVec$, $\GradientVec$ and $\GeodesicCurvature$ rather than $\EXVec$, $\EYVec$ and $\GeodesicRadius$, as following proof.
\begin{proof}[proof of \ThmRef{ExponentialMap}]
First, multiply the numerator and denominator of \EqnRef{GradientComponentUnstable} by $\GeodesicCurvature^2 = \frac{1}{\GeodesicRadius^2}$ and let $\PsiTerm \DefEq \GeodesicCurvature^2 \SPsiTerm^{2} = 1 - \GeodesicCurvature \KY$. 
Now, we have
\begin{equation}
\begin{split}
\GradientComponent = \frac{\KX \Bracket{2 \PsiTerm - \GeodesicCurvature^2 \SPTerm^{2}} - \sqrt{4 \KX^2 \PsiTerm^2 + 4 \PsiTerm^3 \SPTerm^{2} - \GeodesicCurvature^2 \PsiTerm^2 \SPTerm^{4}}}{2 \Bracket{\GeodesicCurvature^2 \KX^2 + \PsiTerm^2}}.
\end{split}
\end{equation}

Here, we can calculate $\PsiTerm$ and $\SPTerm^{2}$ using $\PoleDistance$, $\DepartureDistance$ and $\TwoVecProduct$ as follows:

\begin{equation}
\begin{split}
\PsiTerm
& =
1 - \GeodesicCurvature \KY
\\
& = 1 + \frac{2 \sqrt{\DepartureDistance^2 - \frac{\TwoVecProduct^4}{\PoleDistance^2}}}{\PoleDistance^2 - \DepartureDistance^2} \cdot \frac{\CoshDistanceMinusOne \Paren{\PoleDistance^2 - \DepartureDistance^2}}{2 \PoleDistance^2 + \CoshDistanceMinusOne \Paren{\PoleDistance^2 - \DepartureDistance^2}}
\sqrt{\DepartureDistance^2 - \frac{\TwoVecProduct^4}{\PoleDistance^2}}
\\
& =
1 + \frac{2 \CoshDistanceMinusOne \Paren{\DepartureDistance^2 - \frac{\TwoVecProduct^4}{\PoleDistance^2}}}{2 \PoleDistance^2 + \CoshDistanceMinusOne \Paren{\PoleDistance^2 - \DepartureDistance^2}}
\\
& =
\frac{2 \PoleDistance^2 + \CoshDistanceMinusOne \Paren{\PoleDistance^2 + \DepartureDistance^2 - 2 \frac{\TwoVecProduct^4}{\PoleDistance^2}}}{2 \PoleDistance^2 + \CoshDistanceMinusOne \Paren{\PoleDistance^2 - \DepartureDistance^2}}
\end{split}
\end{equation}

\begin{equation}
\begin{split}
\SPTerm^{2}
& =
\EquidistanceRadius^2 - \Paren{\KX^2 + \KY^2}
\\
& = \EquidistanceRadius^2 - \Abs{\SegmentVec{\Departure}{\EquidistanceCenter}}^2
\\
& =
\frac{\CoshDistanceMinusOne \Paren{\CoshDistanceMinusOne + 2}\PoleDistance^2 \Paren{\PoleDistance^2 - \DepartureDistance^2}^2}{\Bracket{2 \PoleDistance^2 + \CoshDistanceMinusOne \Paren{\PoleDistance^2 - \DepartureDistance^2}}^2} - \Bracket{\frac{\CoshDistanceMinusOne \Paren{\PoleDistance^2 - \DepartureDistance^2}}{2 \PoleDistance^2 + \CoshDistanceMinusOne \Paren{\PoleDistance^2 - \DepartureDistance^2}}}^2 \DepartureDistance^2
\\
& =
\frac{\CoshDistanceMinusOne \Paren{\PoleDistance^2 - \DepartureDistance^2}^2}{2 \PoleDistance^2 + \CoshDistanceMinusOne \Paren{\PoleDistance^2 - \DepartureDistance^2}}
\end{split}
\end{equation}

Define $\SFTerm$, $\SHTerm$ $\SETerm$, $\SZTerm$ and $\SWTerm$ as follows:
\begin{equation}
\begin{split}
\SFTerm
& \DefEq
\frac{\TwoVecProduct^2}{\PoleDistance}
\\
\SHTerm^2
& \DefEq
\PoleDistance^2 - \DepartureDistance^2
\\
\SETerm^3
& \DefEq
\Bracket{2 \PoleDistance^2 + \CoshDistanceMinusOne \Paren{\PoleDistance^2 - \DepartureDistance^2}} \KX
\\
& =
- \CoshDistanceMinusOne \Paren{\PoleDistance^2 - \DepartureDistance^2}
\frac{\TwoVecProduct^2}{\PoleDistance},
\\
& =
- \CoshDistanceMinusOne \SHTerm^2 \SFTerm,
\\
\SZTerm^2
& \DefEq
\Bracket{2 \PoleDistance^2 + \CoshDistanceMinusOne \Paren{\PoleDistance^2 - \DepartureDistance^2}} \PsiTerm
\\
& =
2 \PoleDistance^2 + \CoshDistanceMinusOne \Paren{\PoleDistance^2 - \frac{\TwoVecProduct^4}{\PoleDistance^2}},
\\
& =
2 \PoleDistance^2 + \CoshDistanceMinusOne \Paren{\PoleDistance^2 - \SFTerm^2},
\\
\SWTerm^4
& \DefEq
\Bracket{2 \PoleDistance^2 + \CoshDistanceMinusOne \Paren{\PoleDistance^2 - \DepartureDistance^2}} \SPTerm^{2}
\\
& =
\CoshDistanceMinusOne \Paren{\PoleDistance^2 - \DepartureDistance^2}^2
\\
& =
\CoshDistanceMinusOne \SHTerm^4.
\end{split}
\end{equation}
Using these symbols, we get
\begin{equation}
\begin{split}
\EqnLabel{GradientComponentstable}
\GradientComponent 
& = 
\frac{\SETerm^3 \Bracket{2 \SZTerm^2 - \GeodesicCurvature^2 \SWTerm^4} - \SZTerm^2 \sqrt{4 \SETerm^6 + 4 \SZTerm^2 \SWTerm^4 - \GeodesicCurvature^2 \SWTerm^8}}{2 \Bracket{\GeodesicCurvature^2 \SETerm^2 + \SZTerm^4}}
\\
& =
\SHTerm^2 \frac{- \CoshDistanceMinusOne \SFTerm \Bracket{\SZTerm^2 - 2 \CoshDistanceMinusOne \Paren{\DepartureDistance^2 - \SFTerm^2}} - \SZTerm^2 \sqrt{\CoshDistanceMinusOne \Bracket{\SZTerm^2 - \CoshDistanceMinusOne \Paren{\DepartureDistance^2 - 2 \SFTerm^2}}} }{4 \Paren{\DepartureDistance^2 - \SFTerm^2} \CoshDistanceMinusOne^2 \SFTerm^2 + \SZTerm^4}
\\
& =
\SHTerm^2 \sqrt{\CoshDistanceMinusOne} \SXiTerm,
\end{split}
\end{equation}
where
\begin{equation}
\begin{split}
& \SXiTerm \\
& \DefEq 
\frac{- \sqrt{\CoshDistanceMinusOne} \SFTerm \Bracket{\SZTerm^2 - 2 \CoshDistanceMinusOne \Paren{\DepartureDistance^2 - \SFTerm^2}} - \SZTerm^2 \sqrt{\Bracket{\SZTerm^2 - \CoshDistanceMinusOne \Paren{\DepartureDistance^2 - 2 \SFTerm^2}}} }{4 \Paren{\DepartureDistance^2 - \SFTerm^2} \CoshDistanceMinusOne^2 \SFTerm^2 + \SZTerm^4}.
\end{split}
\end{equation}

Define $\FTerm \DefEq \GradientVec \cdot \DepartureVec$, and $\STTerm \DefEq \frac{\sqrt{\CoshDistanceMinusOne}}{\GradientDistance}$.
We can calculate $\STTerm$ as follows:
\begin{equation}
\begin{split}
\STTerm 
& \DefEq \frac{\sqrt{\CoshDistanceMinusOne}}{\GradientDistance} = \frac{\SHTerm^2}{2 \PoleDistance \sqrt{\cosh \ArrivalDistance + 1}} \frac{\sinh \ArrivalDistance}{\ArrivalDistance} = \frac{\SHTerm^2}{2 \PoleDistance \sqrt{\cosh \ArrivalDistance + 1}} \Sinc \frac{\ArrivalDistance}{\sqrt{-1} \PiUnit}.
\end{split}
\end{equation}

Now, we have $\sqrt{\CoshDistanceMinusOne} \SFTerm = \sqrt{\CoshDistanceMinusOne} \frac{\PoleVec \cdot \DepartureVec}{\PoleDistance} = \sqrt{\CoshDistanceMinusOne} \frac{\GradientVec \cdot \DepartureVec}{\GradientDistance} = \FTerm \STTerm$.
Using these symbols, $\SXiTerm$ can be calculated without using $\TwoVecProduct$ as follows:

\begin{equation}
\begin{split}
& \SXiTerm \\
& =
\frac{- \FTerm \STTerm \Bracket{\SZTerm^2 - 2 \CoshDistanceMinusOne \DepartureDistance^2 + 2 \FTerm^2 \STTerm^2} - \SZTerm^2 \sqrt{\Bracket{\SZTerm^2 - \CoshDistanceMinusOne \DepartureDistance^2 + 2 \FTerm^2 \STTerm^2}} }{4 \DepartureDistance^2 \CoshDistanceMinusOne \FTerm^2 \STTerm^2 - 4 \FTerm^4 \STTerm^4 + \SZTerm^4},
\end{split}
\end{equation}
with
\begin{equation}
\begin{split}
\SZTerm^2
& =
2 \PoleDistance^2 + \CoshDistanceMinusOne \Paren{\PoleDistance^2 + \DepartureDistance^2 - 2 \frac{\TwoVecProduct^4}{\PoleDistance^2}} \\
& =
2 \PoleDistance^2 + \CoshDistanceMinusOne \Paren{\PoleDistance^2 + \DepartureDistance^2} - 2 \FTerm^2 \STTerm^2.
\end{split}
\end{equation}

However, $\NormalComponent$ can still be intractable. 
Recall that
$\SegmentVec{\Departure}{\Arrival} = \GradientComponent \EXVec + \NormalComponent \EYVec$ and $\EXVec = \frac{\PoleVec}{\PoleDistance}$ and $\EYVec$ is given by
\begin{equation}
\EYVec
=
\frac{\DepartureVec - \frac{1}{\PoleDistance^2} \Paren{\DepartureVec \cdot \PoleVec} \PoleVec}{\Abs{\DepartureVec - \frac{1}{\PoleDistance^2} \Paren{\DepartureVec \cdot \PoleVec} \PoleVec}}
=
\frac{\DepartureVec - \frac{\TwoVecProduct^2}{\PoleDistance^2} \PoleVec}{\Abs{\DepartureVec - \frac{\TwoVecProduct^2}{\PoleDistance^2} \PoleVec}}.
.
\end{equation}
Hence,
\begin{equation}
\SegmentVec{\Departure}{\Arrival} = \frac{\GradientComponent}{\PoleDistance} \PoleVec + \frac{\NormalComponent}{\Abs{\DepartureVec - \frac{\TwoVecProduct^2}{\PoleDistance^2} \PoleVec}} \Paren{\DepartureVec - \frac{\TwoVecProduct^2}{\PoleDistance^2} \PoleVec}.
\end{equation}
Therefore, it is sufficient to get $\frac{\NormalComponent}{\Abs{\DepartureVec - \frac{\TwoVecProduct^2}{\PoleDistance^2} \PoleVec}}$ instead of $\NormalComponent$.
We have
\begin{equation}
\Abs{\DepartureVec - \frac{\TwoVecProduct^2}{\PoleDistance^2} \PoleVec}
= \sqrt{\DepartureDistance^2 - \frac{\TwoVecProduct^4}{\PoleDistance^2}}
= \sqrt{\DepartureDistance^2 - \SFTerm^2}
\end{equation}
and
\begin{equation}
\begin{split}
\NormalComponent
& = 
\GeodesicRadius - \sqrt{\GeodesicRadius^2 - \GradientComponent^2}
\\
& =
\frac{\GradientComponent^2}{\GeodesicRadius + \sqrt{\GeodesicRadius^2 - \GradientComponent^2}}
\\
& =
\frac{\GeodesicCurvature \GradientComponent^2}{1 + \sqrt{1 - \GeodesicCurvature^2 \GradientComponent^2}}
\\
& =
\frac{2 \sqrt{\DepartureDistance^2 - \frac{\TwoVecProduct^4}{\PoleDistance^2}}}{\PoleDistance^2 - \DepartureDistance^2}
\frac{\GradientComponent^2}{1 + \sqrt{1 - \GeodesicCurvature^2 \GradientComponent^2}}.
\\
& =
\frac{2 \sqrt{\DepartureDistance^2 - \SFTerm^2}}{\SHTerm^2}
\frac{\GradientComponent^2}{1 + \sqrt{1 - \GeodesicCurvature^2 \GradientComponent^2}}.
\end{split}
\end{equation}
Now, we have
\begin{equation}
\begin{split}
\frac{\NormalComponent}{\Abs{\DepartureVec - \frac{\TwoVecProduct^2}{\PoleDistance^2} \PoleVec}}
& =
\frac{2}{\SHTerm^2} \cdot \frac{\GradientComponent^2}{1 + \sqrt{1 - \GeodesicCurvature^2 \GradientComponent^2}}.
\\
& = \frac{2 \SHTerm^2 \CoshDistanceMinusOne \SXiTerm^2}{1 + \sqrt{1 - 4 \Paren{\DepartureDistance^2 - \SFTerm^2} \CoshDistanceMinusOne \SXiTerm^2}}
\end{split}
\end{equation}
Hence, we get
\begin{equation}
\begin{split}
\SegmentVec{\Departure}{\Arrival} 
& =
\frac{\GradientComponent}{\PoleDistance} \PoleVec + \frac{2 \SHTerm^2 \CoshDistanceMinusOne \SXiTerm^2}{1 + \sqrt{1 - 4 \Paren{\DepartureDistance^2 - \SFTerm^2} \CoshDistanceMinusOne \SXiTerm^2}} \Paren{\DepartureVec - \frac{\TwoVecProduct^2}{\PoleDistance^2} \PoleVec}
\\
& = \Paren{\frac{\SHTerm^2 \sqrt{\CoshDistanceMinusOne} \SXiTerm}{\PoleDistance} - \frac{\TwoVecProduct^2}{\PoleDistance^2} \cdot \frac{2 \SHTerm^2 \CoshDistanceMinusOne \SXiTerm^2}{1 + \sqrt{1 - 4 \Paren{\DepartureDistance^2 - \SFTerm^2} \CoshDistanceMinusOne \SXiTerm^2}}} \PoleVec 
\\ 
& \quad + \frac{2 \SHTerm^2 \CoshDistanceMinusOne \SXiTerm^2}{1 + \sqrt{1 - 4 \Paren{\DepartureDistance^2 - \SFTerm^2} \CoshDistanceMinusOne \SXiTerm^2}} \DepartureVec.
\end{split}
\end{equation}
Recall $\PoleVec = \PoleDistance \frac{\GradientVec}{\GradientDistance}$.
We obtain
\begin{equation}
\begin{split}
\SegmentVec{\Departure}{\Arrival} 
& = \Paren{\SHTerm^2 \STTerm \SXiTerm - \frac{2 \SHTerm^2 \FTerm \STTerm^2 \SXiTerm^2}{1 + \sqrt{1 - 4 \DepartureDistance^2 \CoshDistanceMinusOne \SXiTerm^2 + 4 \FTerm^2 \STTerm^2 \SXiTerm^2}}} \GradientVec 
\\ 
& \quad + \frac{2 \SHTerm^2 \CoshDistanceMinusOne \SXiTerm^2}{1 + \sqrt{1 - 4 \DepartureDistance^2 \CoshDistanceMinusOne \SXiTerm^2 + 4 \FTerm^2 \STTerm^2 \SXiTerm^2}} \DepartureVec.
\end{split}
\end{equation}
\end{proof}


\section{Appendix : Proofs}

\subsection{Hesse operator, strong convexity and smoothness}

The gradient vector field of a function $f \colon \Hyp^\Dimension \to \Real$ gives us the first order information of $f$, and this gives rise to the Riemannian gradient descent algorithms. However, in the context of theoretical analysis, it is useful to consider the second order information of $f$.

\begin{definition}
Given a twice differentiable function $f \colon \Hyp^\Dimension \to \Real$, the Riemannian Hessian $({\rm Hess}\, f)(\MfdPointP)$ at $\MfdPointP$ is defined as a matrix whose $(i,j)$ component is given by
\begin{equation}
\left( \frac{\partial^2 f}{\partial p^i \partial p^j} (\MfdPointP) - \sum_k \Gamma_{ij}^k (\MfdPointP)\frac{\partial f}{\partial p^k}  (\MfdPointP)\right) ,
\end{equation} 
where
\begin{equation}
\Gamma_{ij}^k(\MfdPointP) = \begin{cases}
0 & (i \neq j, j \neq k)\\
\frac{-2p^k}{1-|\MfdPointP|^2} & (i=j, j \neq k),\\
\frac{2p^k}{1-|\MfdPointP|^2} & (i=j=k),\\
\frac{2p^j}{1-|\MfdPointP|^2} &(i=k,i\neq j),\\
\frac{2p^i}{1-|\MfdPointP|^2} &(j=k,i\neq j).
\end{cases}
\end{equation}
\end{definition}

We write $\lambda_{max} (({\rm Hess}\, f)(\MfdPointP))$ as the largest eigenvalue of the matrix, and for any compact subset $K \subset \Hyp^\Dimension$,
\begin{equation}
\lambda_{max} (({\rm Hess}\, f), K) \DefEq \max_{\MfdPointP \in K} \lambda_{max} (({\rm Hess}\, f)(\MfdPointP)).
\end{equation}



The following lemma connects between Hessian tensor and convexity/smoothness of function. For a proof, see \cite{DBLP:journals/tac/Bonnabel13}, for example.

\begin{lemma}
Let $K \subset \Hyp^\Dimension$ be a compact subset, $f \colon K \to \Real$ be a twice differentiable function.
Then $f$ is $\lambda_{min}(({\rm Hess}\, f),K)$-strongly convex and $\lambda_{max}(({\rm Hess}\, f),K)$-smooth.
\end{lemma}

Notice that even if we are working on the same differentiable manifold, the factor of smoothness or convexity varies as the metric is changed.

The following theorem is a consequence of general Riemannian geometry, so we omit the proof. For a proof, see the supplementary A of \cite{1607.02833}, for example.
\begin{theorem}
Let $y \in \Hyp^\Dimension$ and $f(x) = \SqDistance{x}{y}$. The Riemannian hesse operator ${\rm Hess}\, f(x)$ has eigenvalues 1 (with multiplicity 1) and $\theta\, {\rm coth}\, \theta$ (with multiplicity $\Dimension-1$), where $\theta = \Distance{x}{y}$.
\end{theorem}


As a comparison, we calculate the second derivative $\frac{\partial f}{\partial p^i \partial p^j}$ in the case that the function $f$ is $f(\MfdPointP) = \SqDistance{0}{\MfdPointP}$ and $\MfdPointP = (p^1, 0,\cdots,0)^\top$ (This does not lose the generality when we calculate the eigenvalues of the Hessian. If $\MfdPointP$ does not satisfy this condition, rotate the disk in advance). By a direct calculation, we obtain
\begin{equation}
\frac{\partial f}{\partial p^i \partial p^j}  (\MfdPointP) = \begin{cases}
\frac{4\Distance{0}{\MfdPointP}}{|\MfdPointP|-|\MfdPointP|^3} = \frac{2|\MfdPointP|}{1-|\MfdPointP|^2}(\cosh(\Distance{0}{\MfdPointP})-1)& (i = j = 1)\\
\frac{2|\MfdPointP|(4\Distance{0}{\MfdPointP}+1/|\MfdPointP|)}{(1-|\MfdPointP|^2)^2} = \frac{\Distance{0}{\MfdPointP}+1/|\MfdPointP|}{1-|\MfdPointP|^2}(\cosh(\Distance{0}{\MfdPointP})-1)&(i=j\neq 1)\\
0 & (otherwise).
\end{cases}
\end{equation} Therefore we can conclude that the (Euclidean) Hessian Matrix has eigenvalues $\frac{2|\MfdPointP|}{1-|\MfdPointP|^2}(\cosh(\Distance{0}{\MfdPointP})-1)$ (with multiplicity 1) and $\frac{\Distance{0}{\MfdPointP}+1/|\MfdPointP|}{1-|\MfdPointP|^2}(\cosh(\Distance{0}{\MfdPointP})-1)$ (with multiplicity $\Dimension -1$). Therefore we obtain the Proposition \ref{EuclideanRatio}.


\begin{lemma}[A reprint of lemma\ref{smoothness_lemma}]
Let $K \subset \Hyp^\Dimension$ be a compact set, $k_1 = \max_{z \in K} \Distance{0}{z}$, and $k_2 = \max_i\{\Distance{0}{y_i}\}$. Then the function $K \ni x \mapsto \frac{1}{n}\sum_i \SqDistance{x}{y_i}$ is $(k_1+ k_2 + 1)$-smooth.

\begin{proof}
In general, for positive semi-definite matrices $A$ and $B$, the largest eigenvalue $\lambda_{\max}(A+B)$ is smaller than the sum of the largest eigenvalues $\lambda_{\max}(A) + \lambda_{\max}(B)$. In this case, since $\nabla^2 f(x) = \frac{1}{n}\sum \nabla^2 \SqDistance{x}{y_i}$ holds, $\lambda_{\max}(\nabla^2 f(x)) \leq \max_i \{\lambda_{\max}(\nabla^2 \SqDistance{x}{y_i}) \} = \max_i\{ \Distance{x}{y_i} {\rm coth}\, \Distance{x}{y_i} \} \leq \max_i\{ \Distance{x}{y_i} +1 \}$.
\end{proof}
\end{lemma}

\begin{proof}[Proof of Theorem 2]
We give the outline of the proof here. Suppose the initial point $x_0$ is outside of the closed ball of radius $k_2$ centered at the origin. Then, the gradient must be in the direction toward the closed ball, otherwise the value of $f$ increases. Therefore, the sequence will remain inside $K_D$. Now recall that $f$ is 1-strongly convex, and $(k_1 + k_2 + 1)$-smooth inside $K_D$, and apply Theorem 15 of \cite{pmlr-v49-zhang16b}.
\end{proof}

\subsection{Barycenter problems}
Here we give proofs for Theorem \ref{GeodesicBarycenter} and Theorem \ref{NaturalBarycenter}.
\begin{proof}[Proof of Theorem \ref{GeodesicBarycenter}]
In general, for any $\TangentVec \in T_{\MfdPointP}\Hyp^1$ which satisfy $\|\TangentVec\|^2 = d^2$, $\Distance{\MfdPointP}{\ExpMap_{\MfdPointP}(\pm \TangentVec)} = d$ holds. And by direct calculation we have $\|{\rm grad}\, f_0\| = \|{\rm grad}\, f_1\|$ at $\MfdPointP_{opt}$, which carries the result.
\end{proof}

To prove Theorem \ref{NaturalBarycenter}, we begin with the following lemma.
\begin{lemma}
Let $f(\MfdPointR) \DefEq \log\frac{1+\MfdPointR}{1-\MfdPointR}$. For $\MfdPointR \in \Hyp^1$, $\Distance{0}{\MfdPointR} = f(\MfdPointR)$ holds.
\end{lemma}
Using this fact, we can derive that
\begin{equation}
\MfdPointP_{opt} = \frac{1 - \sqrt{(2-\varepsilon)\varepsilon}}{1-\varepsilon}.
\end{equation}

In addition, we need the lemma about the magnitude of a tangent vector in the Euclidean coordinate sense.

\begin{lemma}
Suppose $\TangentVec \in T_{\MfdPointP}\Hyp^1$ satisfies $\|\TangentVec\|=d^2$. Then $|\TangentVec| = \frac{d}{2}\sqrt{1-|\MfdPointP|^2}$.
\end{lemma}
This lemma leads us that
\begin{equation}
\MfdPointR_l = \MfdPointP_{opt} - \StepSize \frac{\sqrt{1-|\MfdPointP_{opt}|^2}}{2}f(\MfdPointP_{opt}), \MfdPointR_r = \MfdPointP_{opt} + \StepSize \frac{\sqrt{1-|\MfdPointP_{opt}|^2}}{2}[f(1-\varepsilon) - f(\MfdPointP_{opt})].
\end{equation}
We put $a = \StepSize \frac{\sqrt{1-|\MfdPointP_{opt}|^2}}{2}f(\MfdPointP_{opt})$, $b = \StepSize \frac{\sqrt{1-|\MfdPointP_{opt}|^2}}{2}f(1-\varepsilon)$. Using this notation, we've just obtained that $\MfdPointR_l = \MfdPointP_{opt} -a$ and $\MfdPointR_r = \MfdPointP_{opt} -a+b$.

\begin{proof}

Let $\MfdPointR_m$ be the hyperbolic middle point of $\MfdPointR_l$ and $\MfdPointR_r$. It is enough to show that $f(\MfdPointR_m) > f(\MfdPointP_{opt}). $ (Notice that $\MfdPointR_m$ and $\MfdPointP_{opt}$ coincide in geodesic case!) Since $\Distance{\MfdPointR_r}{\MfdPointR_l} = \Distance{0}{\MfdPointR_r} - \Distance{0}{\MfdPointR_l} = f(\MfdPointR_r) - f(\MfdPointR_l)$, it is clear that $f(\MfdPointR_m) = \Distance{0}{\MfdPointR_m} = \Distance{0}{\MfdPointR_l} + \frac{1}{2}\Distance{\MfdPointR_l}{\MfdPointR_r} = \frac{1}{2}f(\MfdPointR_l) + \frac{1}{2}f(\MfdPointR_r) = \frac{1}{2}f(\MfdPointP_{opt} - a) + \frac{1}{2}f(\MfdPointP_{opt} -a +b)$. From the convexity of $f$, $f(\MfdPointR_m) \geq f(\MfdPointP_{opt} -a + \frac{1}{2}b)$. Since $f$ is strictly increasing, it is enough to show $b > 2a$, or equivalently, $2f(\MfdPointP_{opt}) < f(1-\varepsilon)$. 

We can verify this by a direct calculation. In general, if $2f(r) < f(1-\varepsilon)$ is satisfied, $(1-\varepsilon)r^2 - 2r +1 > 0$ holds, which is equivalent to that $r$ satisfies $r < \frac{1-\sqrt{\varepsilon}}{1-\varepsilon}$ or $\frac{1+\sqrt{\varepsilon}}{1-\varepsilon} < r$. Since $2-\varepsilon >1$, $\MfdPointP_{opt}$ satisfies this condition.
\end{proof}
\section{Details of Experiments}

\begin{table*}[hbtp]
  \TabLabel{PoincareParameters}
  \caption{Parameters in \Poincare Embedding Experiments.}
  \centering
  \begin{tabular}{lll}
    \toprule
    variable name & value & note \\
    \toprule
    \texttt{size} & 2 & Dimension $\Dimension$ in the body of this paper. \\
    \midrule
    \multirow{3}{*}{\texttt{alpha}} & 0.01, 0.02, 0.05, & \multirow{3}{*}{Learning rate} \\
    & 0.1, 0.2, 0.5, & \\
    & 1.0, 2.0 & \\
    \midrule
    \multirow{2}{*}{\texttt{negative}} & not used & (artificial data) \\
    & 10 & The number of negative samples (real data). \\
    \midrule
    \texttt{epsilon} & 1e-10 & The position of the clipping boundary. \\
    \midrule
    \texttt{regularization\_coeff} & 0 & We did not use regularization. \\
    \midrule
    \texttt{burn\_in} & 0 & We did not use burn in. \\
    \midrule
    \texttt{burn\_in\_alpha} & not used & We did not use burn in. \\
    \midrule
    \texttt{init\_range} & (-0.001, 0.001) & The range of the initial points. \\
    \midrule
    \texttt{dtype} & np.float64 & \\
    \midrule
    \texttt{seed} & 0 & \\
    \bottomrule
  \end{tabular}
\end{table*}

\subsection{Barycenter Problem}
In this section, we give the detail conditions of the Barycenter problem experiments.
\subsubsection{Settings}
We set $n = 2$ and $\MfdPointQ_{1} = (0, 0), \MfdPointQ_{2} = (0, 1 - 1\mathrm{e}-8)$.
We optimized the following function:
\begin{equation}
\sum_{i=1}^{2} \SqDistance{\MfdPointP}{\MfdPointQ_i}.
\end{equation}
\subsubsection{Training}
We optimized the function above by the stochastic descent methods (the Euclidean gradient descent, the natural gradient update, and the geodesic update).
We obtained the stochastic gradient from $\SqDistance{\MfdPointP}{\MfdPointQ_1}$ in probability $\frac{1}{2}$ and from $\SqDistance{\MfdPointP}{\MfdPointQ_2}$ in probability $\frac{1}{2}$.

\subsection{\Poincare Embedding}
\begin{journal}
\subsection{Overall Framework}
\end{journal}
\begin{proceedings}
In this section, we give the detail conditions of the \Poincare embeddings experiments.
\subsubsection{Data Construction}
As a graph, we used complete binary trees (depth$=5$) as synthetic data and noun subset of \emph{WordNet}'s hypernymy relations (subset the root of which is \emph{mammal}) as artificial data.
In the artificial data experiment, we constructed two graphs from the complete binary tree. 
One is the simple undirected graph, which includes both of the edge from each node to its parent and its reverse.
The other is the directed (child to parent) graph with its transitive closure.
Here, the edges from each node to its ancestors including its parent are included, and the edges from each node to its children are not included.
In the real data experiment, we constructed a directed graph in the same way as in \cite{Nickel+:2017}. 
The edges consist of the transitive closure of the hypernymy relations of the nouns. 
For example, as \textit{mammal} is a hypernym of \textit{dog}, directed edge $(\textit{dog} \to \textit{mammal})$ is included in the directed graph. 
Directed edge $(\textit{mammal} \to \textit{animal})$ is also included likewise. Then, directed edge $(\textit{dog} \to \textit{animal})$ is also included. Thus, the directed graph contains hypernymy relations transitively.

\subsubsection{Training}
In the artificial data experiment, we (uniform-randomly) sampled $\Paren{\LeftVec, \PositiveVec} \in \Edges$ for each iteration and obtained the stochastic gradient from the following function:
\begin{equation}
\EqnLabel{StochasticPoincareLossFunction}
- \sum_{\Paren{\LeftVec, \PositiveVec} \in \EdgeSamples} \log \frac{\Exponential{- \Distance{\LeftVec}{\PositiveVec}}}{\sum_{\NegativeNode \in \NegativeNeighborhood{\LeftNode}} \Exponential{- \Distance{\LeftVec}{\NegativeVec}}},
\end{equation}
where $\EdgeSamples \subset \Edges$ is uniformly sampled. 
It is easy to confirm that the expectation of the gradient of \EqnRef{StochasticPoincareLossFunction} is equal to the gradient of \EqnRef{PoincareLossFunction}.

In the real data experiment, we used negative sampling besides the sampling of $\Paren{\LeftVec, \PositiveVec}$.
We uniformly sampled the negative samples $\NegativeSamples{\LeftVec} \subset \NegativeNeighborhood{\LeftVec}$ of $\LeftVec$, and obtained the stochastic gradient from the following function:
\begin{equation}
\EqnLabel{NegativePoincareLossFunction}
- \sum_{\Paren{\LeftVec, \PositiveVec} \in \EdgeSamples} \log \frac{\Exponential{- \Distance{\LeftVec}{\PositiveVec}}}{\sum_{\NegativeNode \in \NegativeSamples{\LeftNode}} \Exponential{- \Distance{\LeftVec}{\NegativeVec}}},
\end{equation}
where $\EdgeSamples \subset \Edges$ and $\NegativeSamples{\LeftNode} \subset \NegativeNeighborhood{\LeftVec}$ is uniformly sampled.
Note that the expectation of the gradient of \EqnRef{StochasticPoincareLossFunction} is no longer equal to the gradient of \EqnRef{PoincareLossFunction}.
Hence, the optimization using the oracle on the basis of \EqnRef{NegativePoincareLossFunction} does not optimize the original loss function \EqnRef{PoincareLossFunction}. 
Therefore, it is difficult to evaluate the methods using the value of the original loss function.

\subsection{Parameter Settings}
\TabRef{PoincareParameters} shows the parameter settings in the \Poincare embeddings experiments.
\end{proceedings}

\begin{journal}
\subsection{Datasets}
Experiments in this section is similar to ones in \cite{Nickel+:2017}; we constructed a directed graph from \emph{WordNet}, which is a knowledge base containing words and phrases with relations among them such as symnonymy and hypernymy relations.
We constructed a directed graph such that the nodes consists of the nouns in \emph{WordNet} and the edges consists of the transitive closure of the hypernymy relations of the nouns. For example, since \textit{mammal} is a hypernym of \textit{dog}, directed edge $(\textit{dog} \to \textit{mammal})$ is included in the directed graph. Directed edge $(\textit{mammal} \to \textit{animal})$ is also included likewise. Then, directed edge $(\textit{dog} \to \textit{animal})$ is also included. Thus, the directed graph contains hypernymy relations transitively.
\subsection{Embeddings and Optimization}
We embedded the nouns into the hyperbolic space by minimizing the loss function in \cite{Nickel+:2017}. 
Let $\Vertices$ be the nouns in \emph{WordNet} and let $\Edges \subset \Vertices \times \Vertices$ be the observed hypernymy relations between noun pairs in \emph{WordNet}.
Define positive neighborhood $\PositiveNeighborhood{\LeftVec} \subset \Edges$ of $\LeftVec$ by $\PositiveNeighborhood{\LeftVec} \DefEq \SetBuilder{\PositiveVec \in \Vertices}{\Paren{\LeftVec, \PositiveVec} \in \Edges}$, and define negative neighborhood $\NegativeNeighborhood{\LeftVec} \subset \Vertices \times \Vertices$ of $\LeftVec$ by $\NegativeNeighborhood{\LeftVec} \DefEq \SetBuilder{\PositiveVec \in \Vertices}{\Paren{\LeftVec, \PositiveVec} \in \Edges}$.
The loss function is defined as follows:
\begin{equation}
\LossFunc{\Params} \DefEq \sum_{\LeftVec} \PointLossFunc{\LeftVec}{\Params},
\end{equation}
where
\begin{equation}
\PointLossFunc{\LeftVec}{\Params} 
\DefEq - \sum_{\PositiveVec \in \PositiveNeighborhood{\LeftVec}} \log \frac{\Exponential{- \Distance{\LeftVec}{\PositiveVec}}}{\sum_{\NegativeVec \in \NegativeNeighborhood{\LeftVec}} \Exponential{- \Distance{\LeftVec}{\NegativeVec}}}.
\end{equation}
For training, we randomly sample $\LeftVec$ and negative samples $\NegativeSamples{\LeftVec} \subset \NegativeNeighborhood{\LeftVec}$, and update $\LeftVec$ and $\PositiveVec$ for all $\PositiveVec \in \PositiveNeighborhood{\LeftVec}$ and $\NegativeVec$ for all $\NegativeVec \in \NegativeSamples{\LeftVec}$.
Here, we apply our geodesic update rule or the natural gradient method proposed in \cite{Nickel+:2017} implemented in \emph{gensim} \cite{Rehurek:2010}.
\subsection{Evaluation Methods}
We evaluate the embeddings by our method and the compared methods throughout the value of the loss function and the tasks listed below:
\begin{description}
\item[Reconstruction] We ranked each relation $\Paren{\LeftVec, \PositiveVec} \in \Edges$ of the observed relations $\Edges$ among the ground truth negative example relations $\SetBuilder{\Paren{\LeftVec, \NegativeVec}}{\NegativeVec \in \NegativeSamples{\LeftVec}}$ of $\LeftVec$. 
Then, we evaluated the mean rank and the mean average precision of the ranking.
In this experiment, the capacity of the embedded data to reconstruct the original graph was evaluated. 
\item[Link Prediction] We split the edges $\Edges$ into a train data $\Edges_\mathrm{train}$ and test data $\Edges_\mathrm{test}$, and embedded the nouns by minimizing loss function defined by $\Edges_\mathrm{train}$. 
After training, We ranked each relation $\Paren{\LeftVec, \PositiveVec} \in \Edges_\mathrm{test}$ of the observed relations $\Edges$ among the ground truth negative example relations $\SetBuilder{\Paren{\LeftVec, \NegativeVec}}{\NegativeVec \in \NegativeSamples{\LeftVec}}$ of $\LeftVec$. Then, we evaluated the mean rank and the mean average precision of the ranking.
In this experiment, the generalization performance of the embedded data to predict the original graph was evaluated.
\end{description}
For all the tasks above, we used the implementation in \emph{gensim} \cite{Rehurek:2010}.
\TabRef{Loss} shows the loss function as a result of optimization with Riemannian gradient method and our geodesic-based method. Our method achieved lower values of loss function in each setting. On the other hand, as \TabRef{Reconstruction}, and \TabRef{Prediction} shows, our method gave worse scores in reconstruction task and prediction task. This is because loss function used in this task does not directly optimize the reconstruction performance or prediction performance. In other words, the loss function is, as it were, a relaxed version of that of the original tasks, and optimization of the loss function does not always means the optimization of the performance in the reconstruction task or prediction task.
\end{journal}

\end{document}